\documentclass[letterpaper]{article} 
\usepackage{aaai24}  
\usepackage{times}  
\usepackage{helvet}  
\usepackage{courier}  
\usepackage[hyphens]{url}  
\usepackage{graphicx} 
\usepackage{natbib}  
\usepackage{caption} 
\usepackage{algorithm}
\usepackage{algorithmic}

\usepackage{newfloat}
\usepackage{listings}
\DeclareCaptionStyle{ruled}{labelfont=normalfont,labelsep=colon,strut=off} 
\floatstyle{ruled}
\newfloat{listing}{tb}{lst}{}
\floatname{listing}{Listing}
\usepackage{amsmath}
\usepackage{subfigure}
\usepackage{algorithm}
\usepackage{algorithmic}
\usepackage{multirow}
\usepackage{array}
\usepackage{booktabs}
\usepackage{natbib}
\usepackage{mathtools}
\usepackage{amsfonts}
\usepackage{dsfont}

\newcommand\myeq{\stackrel{\mathclap{\normalfont{D}}}{=}}
\newcommand*\circled[1]{\tikz[baseline=(char.base)]{		\node[shape=circle,fill=black,text=white,draw,inner sep=.1pt] (char) {#1};}}
\usepackage{tikz}

\newcommand*{\dif}{\mathop{}\!\mathrm{d}}
\providecommand{\tabularnewline}{\\}

\newcommand{\BlackBox}{\rule{1.5ex}{1.5ex}}  
\ifdefined\proof
    \renewenvironment{proof}{\par\noindent{\bf Proof\ }}{\hfill\BlackBox\\[2mm]}
\else
    \newenvironment{proof}{\par\noindent{\bf Proof\ }}{\hfill\BlackBox\\[2mm]}
\fi
 
\newtheorem{theorem}{Theorem}
\newtheorem{lemma}[theorem]{Lemma} 
\newtheorem{proposition}[theorem]{Proposition}

\newtheorem{definition}[theorem]{Definition}

\title{OVD-Explorer: \\Optimism Should Not Be the Sole Pursuit of Exploration in Noisy Environments}
\author{
    Jinyi Liu\textsuperscript{\rm 1}\equalcontrib, Zhi Wang\textsuperscript{\rm 2}\equalcontrib, Yan Zheng\textsuperscript{\rm 1}\thanks{Corresponding author: Yan Zheng (yanzheng@tju.edu.cn)}, Jianye Hao\textsuperscript{\rm 1}, Chenjia Bai\textsuperscript{\rm 3}, Junjie Ye\textsuperscript{\rm 2}, \\Zhen Wang\textsuperscript{\rm 4}, Haiyin Piao\textsuperscript{\rm 4}, Yang Sun\textsuperscript{\rm 5} \\
}
\affiliations{
     \textsuperscript{\rm 1}{College of Intelligence and Computing, Tianjin University}, \textsuperscript{\rm 2}{Independent Researcher},  \\ \textsuperscript{\rm 3}{Shanghai AI Laboratory}, \textsuperscript{\rm 4}{Northwestern Polytechnical University}, \textsuperscript{\rm 5}{SADRI Institute} \\
    \{jyliu, yanzheng, jianye.hao\}@tju.edu.cn, zhiwoong@163.com, baichenjia@pjlab.org.cn, kourenmu@gmail.com, \\\{w-zhen, haiyinpiao\}@nwpu.edu.cn, yang.sun2010@gmail.com

}

\begin{document}

\maketitle

\begin{abstract}
In reinforcement learning, the optimism in the face of uncertainty (OFU) is a mainstream principle for directing exploration towards less explored areas, characterized by higher uncertainty. However, in the presence of environmental stochasticity (noise), purely optimistic exploration may lead to excessive probing of high-noise areas, consequently impeding exploration efficiency. Hence, in exploring noisy environments, while optimism-driven exploration serves as a foundation, prudent attention to alleviating unnecessary over-exploration in high-noise areas becomes beneficial. In this work, we propose Optimistic Value Distribution Explorer (OVD-Explorer) to achieve a noise-aware optimistic exploration for continuous control. OVD-Explorer proposes a new measurement of the policy's exploration ability considering noise in optimistic perspectives, and leverages gradient ascent to drive exploration. Practically, OVD-Explorer can be easily integrated with continuous control RL algorithms.  Extensive evaluations on the MuJoCo and GridChaos tasks demonstrate the superiority of OVD-Explorer in achieving noise-aware optimistic exploration.
\end{abstract}
\section{Introduction}
Efficient exploration is crucial for improving the reinforcement learning (RL) efficiency and ultimate policy performance \citep{sutton2018reinforcement}, and many exploration strategies have been proposed in the literatures~\citep{LillicrapHPHETS15,DBLP:conf/nips/OsbandBPR16,chen2017ucb,DBLP:conf/nips/CiosekVLH19}. Most of them follows the \textit{Optimism in the Face of Uncertainty}~(OFU) principle \citep{AuerCF02} to guide exploration optimistically towards the area with high uncertainty \citep{chen2017ucb, DBLP:conf/nips/CiosekVLH19}. 
Conceptually, OFU-based methods regard the uncertainty as the ambiguity caused by insufficient exploration, and is high at those state-action pairs seldom visited, referred to as \textit{epistemic uncertainty} \citep{DBLP:conf/nips/OsbandBPR16}.

Another kind of uncertainty existed in RL is known as \textit{aleatoric uncertainty}, caused by the randomness in the environment or policy, and referred to as \textit{noise}~\citep{KirschnerK18} or risk~\citep{DabneyOSM18}. The noise is ubiquitous in real world. 
For example, unpredictable wind shifts the trajectory after an robot's action, and rough ground changes the force point of objects, etc. 
However, overly visiting such noisy areas may cause
severely unstable state transitions (the \textit{Optimistic} arrow in the intuitive example in  Fig.~\ref{fig:intuitive_example}), thus is detrimental to the learning efficiency \citep{abs-1905-09638}. 
For this, risk-averse policy is proposed to avoid visiting the areas with high aleatoric uncertainty estimation \citep{DabneyRBM18QRDQN, DabneyOSM18}.
Typical approaches use Conditional Variance at Risk (CVaR) to calculate a conservative value estimation and guide policy learning for easing the negative effect of the noise \citep{DabneyOSM18}.
However, indiscriminately avoiding noise may also yield no performance guarantee due to excessively conservation (the \textit{risk-averse} arrow in Fig.~\ref{fig:intuitive_example}).

\label{sec:intro}
\begin{figure}[t!]
    \centering
    \includegraphics[width=0.37\textwidth]{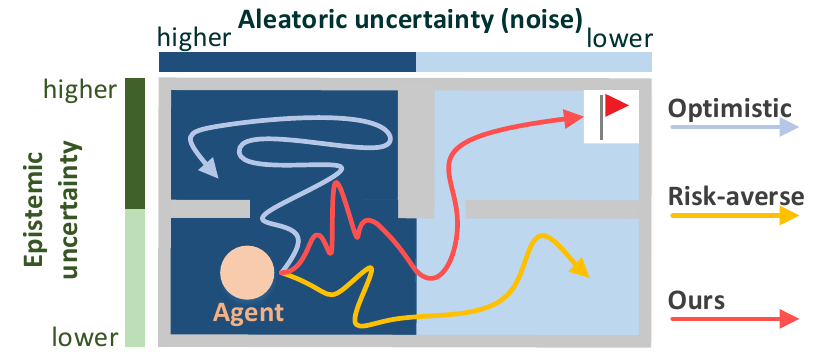}
    \caption{An intuitive example. The agent learns to move to the flag in a room filled with noise, and the noise in the left side is higher. The \textit{optimistic} exploration strategy may overly explore the noisy area, and the \textit{risk-averse} policy indiscriminately avoids noisy areas but explores insufficiently. 
    }
    \label{fig:intuitive_example}
\end{figure}

Therefore, a more reasonable approach is integrating the risk-averse policy and optimistic exploration, guiding the agent to optimistically exploring the whole areas, while avoiding overly exploring the areas with high noise, 
like the \textit{Ours} arrow in Fig.~\ref{fig:intuitive_example}. Note that overly exploring noisy areas may damage performance, but moderate exploration of such area is necessary, which should be ensured by the ability of optimistic exploration. 
The similar concern has been demonstrated to be effective under discrete control tasks \citep{DBLP:conf/iclr/NikolovKBK19, abs-1905-09638}.
However, for continuous control tasks, such concern has not yet been investigated well. 
A natural way to apply discrete control algorithms to solve continuous control problems is to discretize the continuous action space, but it suffers from the scalability issue due to the exponentially increasing discretized actions \citep{AntosMS07,hyar}, and it may throw away crucial information in action space causing its performance to be compromised 
\citep{LillicrapHPHETS15,Tang020Discretizing}. 
Thus, designing such an optimistic exploration strategy that can avoid overly exploring noisy area for continuous action space is required.

In this work, we propose OVD-Explorer, a noise-aware optimistic exploration strategy that applies to continuous control tasks for the first time. 
Specifically, we propose a new policy's exploration ability measurement, quantifying both the ability of avoiding noise and pursuing optimisticity during exploration. 
To capture the noise, the value distribution is modelled. Further, to guide optimistic exploration, the upper bound distribution of return is approximated using Optimistic Value Distribution (OVD), representing the best returns that the policy can reach. Then we quantitatively measure such ability using OVD, and generates the behavior policy by maximizing the exploration ability measurement, thus names our approach as Optimistic Value Distribution Explorer (OVD-Explorer).

To make OVD-Explorer tractable for continuous control, we generate the behavior policy using a gradient-based approach, and propose a scheme to incorporate it with policy-based RL algorithms. Practically, we demonstrate the exploration benefits based on SAC~\cite{HaarnojaZAL18SAC}, a well-performed continuous RL algorithm.
Evaluations on various continuous RL tasks, including the GridChaos, MuJoCo tasks and their stochastic version, are conducted. The results demonstrate the effectiveness of OVD-Explorer in achieving an optimistic exploration that avoids overly exploring noisy areas, leading to a better performance.

\section{Related Work}
\label{app:related}
In this work, we consider the exploration strategy under the OFU principle~\citep{AuerCF02}, 
and aim to a noise-aware optimistic exploration strategy. 

\textit{Overview of exploration approaches.} Basic exploration strategies always lead to undirected exploration through random perturbations~\citep{LillicrapHPHETS15, sutton2018reinforcement, HaarnojaZAL18SAC}. 
With the increasing emphasis on exploration efficiency in RL, various exploration methods have been developed~\citep{hao2023exploration}. One kind of methods uses intrinsic motivation to stimulate agent to explore~\citep{MartinSEH17, BellemareSOSSM16, SavinovRVMPLG19, houthooft2016vime, Badia20NGU, YuanHNMZHLCF23}. 
Some other methods, originating from tracking uncertainty, guide exploration under the OFU principle~\citep{thompson1933likelihood, DBLP:conf/nips/OsbandBPR16, DBLP:conf/iclr/NikolovKBK19, DBLP:conf/nips/CiosekVLH19, PathakG019,bai-OB2I,bai-DB}. 
The key of OFU-based exploration methods is modeling the epistemic uncertainty~\citep{DBLP:conf/nips/OsbandBPR16, GalG16dropout,bai-UCB}. Specifically, we use ensemble~\citep{DBLP:conf/nips/OsbandBPR16} for estimating epistemic uncertainty. 

\textit{The issue of overly exploring noisy areas.} There is another uncertainty in RL system, i.e., aleatoric uncertainty~(a.k.a. noise~\citep{KirschnerK18} or risk~\citep{DabneyOSM18}), captured by return distribution \citep{BellemareDM17, DabneyOSM18, DabneyRBM18QRDQN}. Overly exploring the areas with high noise could make learning unstable and inefficient, thus many works seek a conservative and noise-averse (or risk-averse) policy to make the policy stable \citep{DabneyOSM18, dsac,bai-MQN}. Nevertheless, conservative alone without advanced exploration could induce low exploration efficiency, and exploration without avoiding noise could make interaction risky. Thus some recent works produce optimistic exploration strategies considering risk~\citep{DBLP:conf/icml/MavrinYKWY19, DBLP:conf/iclr/NikolovKBK19}. 
However, such methods are complicated when deriving a behavior policy and only limited to discrete control. 

Indeed, addressing noise in exploration poses a challenge for well-performing continuous RL algorithms~\citep{HaarnojaZAL18SAC, dsac}. While exploration strategies like OAC \citep{DBLP:conf/nips/CiosekVLH19} are designed following OFU principle, guided by the upper bound of $Q$ estimation, they overlook the potential impact of noise. This oversight can lead to misguided exploration, hampering the learning process. To address that, we propose OVD-Explorer to guide agent to explore optimistically, while avoiding overly exploring the noisy areas, improving the robustness of exploration especially facing heteroscedastic noise.

\section{Preliminaries}
\label{sec:pre}

\subsection{Distributional Value Estimation}
To capture the environment noise, we use quantile regression \citep{DabneyRBM18QRDQN} to formulate $Q$-value distribution. 
$Q$-value distribution, represented by the quantile random variable $Z$, maps the state-action pair to a uniform probability distribution supported on the return values at all corresponding quantile fractions. 
Given state-action pair $(s, a)$, we denote the $i$-th quantile fraction as $\tau_i$, and the value at $\tau_i$ as $Z_{\tau_i}(s, a)$, where $\tau_i \in [0, 1]$.

Based on the Bellman operator \citep{WatkinsD92}, the distributional Bellman operator \citep{BellemareDM17} $\mathcal{T}_D^\pi$ under policy $\pi$ is given as:
\begin{equation}
\begin{aligned}
\label{eq:dbo}
\mathcal{T}_D^\pi Z(s, a) \myeq R(s, a) + \gamma Z(s', a'), a' \sim \pi(\cdot | s').
\end{aligned}
\end{equation}
Notice that $\mathcal{T}_D^\pi$ operates on random variables, $\myeq$ denotes that distributions on both sides have equal probability laws. Based on operator $\mathcal{T}_D^\pi$, QR-DQN~\citep{DabneyRBM18QRDQN} trains quantile estimations via the quantile regression loss \citep{koenker2001quantile}, which is denoted as:
\begin{equation}
\label{3-2-e2}
\mathcal{L}_{QR}(\theta) = \frac{1}{N} \sum_{i=1}^{N}\sum_{j=1}^{N}[\rho_{\hat{\tau}_i}(\delta_{i, j})],
\end{equation}
where $\theta$ and $\bar{\theta}$ is the parameters of the value distribution estimator and its target network, respectively, TD error $\delta_{i, j} = R(s, a) + \gamma {Z}_{\hat{\tau}_i}(s', a'; \bar{\theta}) - {Z}_{\hat{\tau}_j}(s, a; \theta)$, the quantile Huber loss $\rho_{\tau}(u) = u * |\tau - \mathds{1}_{u < 0}|$, and $\hat{\tau}_i$ means the quantile midpoints, which is defined as $\hat{\tau}_i = \frac{\tau_{i-1} + \tau_{i}}{2} $.
\subsection{Distributional Soft Actor-Critic}
\label{sec3.3}
Distributional Soft Actor-Critic (DSAC) \citep{dsac} seamlessly integrates distributional RL with Soft Actor-Critic (SAC) \citep{HaarnojaZAL18SAC}.
Basically, based on the Eq.~\ref{eq:dbo},
the distributional soft Bellman operator $\mathcal{T}^\pi_{DS}$ is defined considering the maximum entropy RL as follows:
\begin{equation}
\begin{aligned}
\mathcal{T}^\pi_{DS} Z(s, a)\ \myeq\  R(s, a) + \gamma [Z(s', a') - \alpha \log \pi(a' | s')],
\end{aligned}
\end{equation}
where $a' \sim \pi(\cdot | s'), s' \sim \mathcal{P}(\cdot|s, a)$.
Then, to overcome overestimation of $Q$-value estimation, DSAC extends clipped double $Q$-Learning \citep{FujimotoHM18}, maintaining two critic estimators $\theta_k, k=1, 2$. 
Thus, the quantile regression loss differs from Eq.~\ref{3-2-e2} on TD loss of $\theta_l$:
\begin{equation}
\begin{aligned}
\label{3-3-e5}
 \delta_{i, j}^l = &R(s, a) + \gamma [\min_{k=1, 2} {Z}_{\hat{\tau}_i}(s', a'; \bar{\theta}_k) - \alpha \log \pi(a' | s'; \bar{\phi})] \\&- {Z}_{\hat{\tau}_j}(s, a; \theta_l),
\end{aligned}
\end{equation}
where $\bar{\theta} $ and $\bar{\phi}$ represents their target networks respectively. The objective of actor is the same as SAC,
\begin{equation}
\begin{aligned}
\label{3-3-e7}
    \mathcal{J}_\pi(\phi) = \mathop{\mathbb{E}}\limits_{\substack{s\sim \mathcal{D}\\ \epsilon \sim \mathcal{N}}}[\ \log \pi(f_\phi(s, \epsilon) | s) - Q(s, f_\phi(s, \epsilon); \theta)\ ],
\end{aligned}
\end{equation}
where $\mathcal{D}$ is the replay buffer, $f_\phi(s, \epsilon)$ means sampling action with re-parameterized policy and $\epsilon$ is a noise vector sampled from any fixed distribution, like standard spherical Gaussian. Here, $Q$ value is the minimum value of the expectation on totally $N$ quantile fractions, as
\begin{equation}
\begin{aligned}
    Q(s, a; \theta) = \min_{k=1, 2} {\mathbb{E}}_{i \sim \mathcal{U}(1, N)}{Z}_{\hat{\tau}_i}(s, a; \theta_k).
\end{aligned}
\end{equation}

\section{Optimistic Value Distribution Explorer}
\label{sec:OVD-Explorer}

In noisy environments, a more efficient exploration strategy entails being noise-aware optimistic, especially to avoid excessive exploration in noisy areas. 
Over exploration towards the areas with high noise may damage the exploration performance, but indiscriminately avoiding visiting such areas could also compromise performance due to excessively conservation and insufficient exploration.
In this work, we propose OVD-Explorer to achieve a noise-aware optimistic exploration in continuous RL. Accordingly, the key insight, theoretical derivation and formulation, and analysis of OVD-Explorer are outlined below. 

\subsection{Noise-aware Optimistic Exploration}
\label{sec:3-1}

Several previous optimistic exploration strategies for continuous control typically estimate the upper bound of $Q$-value, and guide exploration by maximizing this upper bound \citep{DBLP:conf/nips/CiosekVLH19, LeeLSA21}.
While such upper bounds provide valuable guidance for optimistic exploration, they fail to capture the noise in the environment. To address that, we propose to incorporate the value distribution into the definition of the upper bound to capture noise, and define the upper bound distribution of $Q$-value. Additionally, we introduce a novel exploration ability measurement for policy distribution $\pi(\cdot)$ using such upper bound distribution, to characterize a policy's ability for noise-aware optimistic exploration. We then derive the behavior (exploration) policy by maximizing this ability measurement.

Firstly, we define the upper bound distribution of the $Q$-value at each state-action pair as $\Bar{Z}^\pi(s,a)$.
\begin{definition}[The upper bound distribution of $Q$-value]
Given state-action pair $(s, a)$, the upper bound distribution of its $Q$-value, denoted as $\Bar{Z}^\pi(s,a)$, is a value distribution satisfying that at each quantile fraction $\tau_i \in [0, 1]$, its value $\Bar{Z}^\pi_{\tau_i}(s, a)$ is the upper bound of possible estimations:
\begin{equation}
\Bar{Z}^\pi_{\tau_i}(s, a) := \sup_\theta {Z}^\pi_{\tau_i}(s, a; \theta),
\end{equation}
where $\theta$ represents different estimators of value distribution, ${Z}^\pi_{\tau_i}(s, a; \theta)$ represent the value at quantile fraction $\tau_i$ of the value distribution estimation ${Z}^\pi(s, a; \theta)$.
\end{definition}

We expect an effective exploration policy to approach the upper bound of $Q$-value. With the distribution-based definition of such an upper bound, we then employ mutual information to evaluate the correlation between the policy distribution and the upper bound distribution, which forms the basis for our definition of exploration capability. Overall, given current state $s$, we quantitatively measure the policy's exploration ability, denoted as ${\bf{F}}^\pi(s)$, by the integral of mutual-information between policy $\pi(\cdot | s)$ and the upper bound distributions of $Q$-value over the action space:
\begin{equation}
\label{eq:def_mi}
    {\bf{F}}^\pi(s)
    = \int_{a'} {\bf{MI}}(\Bar{Z}^\pi(s,a'); \pi(\cdot|s)|s) \dif{a'}
\end{equation}
where $a' \in \bf{A}$ denotes any legal action.
Now we state how to approximate the exploration ability in Proposition~\ref{the}.
\begin{proposition}
\label{the}
The mutual information in Eq.~\ref{eq:def_mi} at state $s$ can be approximated as:
\begin{equation}
\begin{aligned}
\scriptsize
\label{eq4-1-5}
{\bf{F}}^\pi(s) \approx
\frac{1}{C} \mathop{\mathbb{E}}_{\substack{a\sim\pi(\cdot | s)\\ \bar{z}(s, a)\sim \bar{Z}^\pi(s, a)}}\left[\Phi_{Z^\pi}(\bar{z}(s, a)) \log \frac{\Phi_{Z^\pi}(\bar{z}(s, a))}{C}  \right].
\end{aligned}
\end{equation}
$\Phi_{x}(\cdot)$ is the cumulative distribution function (CDF) of random variable $x$, $\bar{z}(s,a)$ is the sampled upper bound of return from its distribution $\Bar{Z}^\pi(s,a)$ following policy $\pi$, $Z^{\pi}$ describes the current return distribution of the policy $\pi$, and $C$ is a constant~(see proof in App.~\ref{proof:themi}).
\end{proposition}

Note that, to optimize the above objective, we need to formulate two critical components at any state-action pair $(s, a)$ under policy $\pi$: \circled{1} the return distribution $Z^\pi(s,a)$ and \circled{2} the upper bound distribution of return $\bar{Z}^\pi(s,a)$. We detail the formulations in Sec.~\ref{sec4.2}.2. 

Proposition~\ref{the} reveals that ${\bf{F^{\pi}}}(s)$ is only proportional to the CDF value {$\Phi_{Z^\pi}(\bar{z}(s, a))$}, which is also proportional to $\bar{z}(s, a)$, an upper bound of $Q$-value, thus a higher ${\bf{F^{\pi}}}(s)$ represents the higher ability of optimistic exploration, following traditional OFU principle. 
Meanwhile, $\Phi_{Z^\pi}(\bar{z}(s, a))$ increases as the variance of current return distribution $Z^\pi$ becomes lower, thus the higher ${\bf{F^{\pi}}}(s)$ means the higher ability of exploring towards the areas with low variance of return distribution, i.e., low noise. A more detailed analysis is given in Sec.~\ref{sec4.4}.3.

Given current state $s$, OVD-Explorer aims to find the behavior policy $\pi_E$ which has the best exploration ability ${\bf{F^{\pi}}}(s)$ in the policy space $\Pi$, as follows:
\begin{equation}
\label{eq4-0-1}
   \pi_E = \arg\max_{\pi \in \Pi} {\bf{F}}^\pi(s).
\end{equation}
For continuous action space, generating the analytical solution $\pi_E$ in Eq.~\ref{eq4-0-1} is intractable. Hence, we propose to perform the gradient ascent based on the policy $\pi$, so as to iteratively deriving a behavior policy with high ability of noise-aware optimistic exploration.
In short, given the policy $\pi_\phi$ parameterized by $\phi$, we calculate the derivative $\nabla_\phi \Phi_{Z^{\pi_\phi}}(\bar{z}(s, a))$ and guide $\phi$ along the gradient direction to improve the exploration ability (more details in Sec.~\ref{4-3}).

\subsection{Distributions of Return's Upper Bound and Return}
\label{sec4.2}

Now, we introduce the formulation of the return distribution $Z^\pi(s,a)$ and its upper bound distribution $\bar{Z}^\pi(s,a)$.

In specific, we use two value distribution estimators $\hat{Z}(s,a; \theta_1)$ and $\hat{Z}(s,a; \theta_2)$ parameterized by $\theta_1$ and $\theta_2$, as ensembles to formulate $\bar{Z}^\pi$ and $Z^\pi$ differently. Unless stated otherwise, $(s,a)$ is omitted hereafter to ease notation. As mentioned earlier, two types of uncertainties are involved, epistemic uncertainty and aleatoric uncertainty (noise), denoted as $\sigma_{\text{epistemic}}^2(s, a)$ and $\sigma_{\text{aleatoric}}^2(s, a)$, respectively. Due to space limitations, the computational details regarding uncertainty value are presented in detail in the appendix.

\textit{Formulation of $\bar{Z}^\pi$.}
The $\bar{Z}^\pi$ denotes the upper bound distribution of return that policy $\pi$ can reach. We propose Gaussian distribution with optimistic mean value $\mu_{\bar{Z}}(s,a)$ for formulation to formulate $\bar{Z}^\pi(s, a)$ as follows, and accordingly refer to it as \textit{Optimistic Value Distribution} (OVD):
\begin{equation}
\label{4-2-z*}
    \bar{Z}^\pi(s,a) \sim \mathcal{N} (\mu_{\bar{Z}}(s,a), \sigma^2_{\text{aleatoric}}(s,a)),
\end{equation}
where $\sigma^2_{\text{aleatoric}}(s,a)$ is its variance.
Notable, \citet{chen2017ucb} discovers the optimisticity is beneficial for better estimating the upper bound, which motivates us to optimistically estimate $\mu_{\bar{Z}}(s,a)$ as averaged upper bound value of return by considering epistemic uncertainty as follows:
\begin{equation}
\begin{aligned}
\label{4-2-e13}
&\mu_{\bar{Z}}(s,a) = \mu(s,a) + \beta \sigma_{\text{epistemic}}(s, a),\ \\& \textup{s.t. } \mu(s,a) = \mathbb{E}_{i \sim \mathcal{U}(1, N)}\mathbb{E}_{k={1, 2}} \hat{Z}_{\tau_i}(s, a; \theta_k)
\end{aligned}
\end{equation}
where $\mu(s,a)$ represents the expected $Q$-value estimation, and uncertainty value is weighted by $\beta$, $\mathcal{U}$ is uniform distribution, $N$ is the number of quantiles, and $\hat{Z}_{\tau_i}(s,a;\theta_k)$ is the value of the $i$-th quantile drawn from $\hat{Z}(s,a;\theta_k)$.
 
Leveraging optimistic value estimations together with explicitly modeling the noise, the upper bound distribution $\bar{Z}^\pi$ can be comprehensively formulated, known as OVD. Such an optimistic distribution can guides effectively optimistic exploration for OVD-Explorer.

\textit{Formulation of $Z^{\pi}$.}
$Z^\pi$ estimates the return distribution obtained following policy $\pi$. Following \citet{FujimotoHM18}, to alleviate overestimation, we formulate $Z^\pi$ in a pessimistic way.
In practice, $Z^\pi$ can be measured in two ways. First, similar to formulating $\bar{Z}^\pi$ in Eq.~\ref{4-2-z*}, $Z^\pi$ can also be formulated as Gaussian distribution as follows:
\begin{equation}
\begin{aligned}
\label{4-2-e15}
    &Z^{\pi}(s,a) \sim \mathcal{N}(\mu_{Z^\pi}(s,a), \sigma^2_{\text{aleatoric}}(s, a)), \quad \\&\textit{s.t.}\quad \mu_{Z^\pi}(s,a) = \mu(s, a) - \beta \sigma_{\text{epistemic}}(s, a),
\end{aligned}
\end{equation}
where $\mu(s, a)$, $\sigma_{\text{aleatoric}}(s, a)$ and $\sigma_{\text{epistemic}}(s, a)$ are the same defined in Eq.~\ref{4-2-e13}. Differently, $\sigma_{\text{epistemic}}(s, a)$ is subtracted from $\mu(s,a)$ to reveal the pessimistic estimation.

Another way is to formulate $Z^\pi$ pessimistically as multivariate uniform distribution as:
\begin{equation}
\label{4-2-e17}
\begin{aligned}
&Z^{\pi}(s,a) \sim \mathcal{U}\{z^{\pi}_i(s,a;\theta)\}_{i=1, ..., N}, \quad \\&\textit{s.t.}\ z^{\pi}_i(s,a;\theta)=\min_{k = {1, 2}} \hat{Z}_{\tau_i}(s, a; \theta_k),
\end{aligned}
\end{equation}
where each quantile value $z^{\pi}_i(s,a;\theta) $ is the minimum estimated value among ensemble estimators (i.e., $\hat{Z}_{\tau_i}(s,a;\theta_k$).

OVD-Explorer formulates the value distribution $Z^{\pi}$ in two ways using Eq.~\ref{4-2-e15} and Eq.~\ref{4-2-e17}, abbreviated in the   following as OVDE\_G and OVDE\_Q, respectively.
Intuitively, Gaussian distribution is expected to help more when the environment randomness follows a unimodal distribution, and multivariate uniform distribution is more flexible and suitable for scenarios with multi-modal distributions. 

\subsection{Analysis of OVD-Explorer}
\label{sec4.4}
To analyzes how OVD-Explorer \textit{optimistically} explores the whole areas and performs \textit{noise-aware} exploration at the same time, an intuitive example involving two actions is adopted.
According to Proposition~\ref{the}, the behavior policy of OVD-Explorer maximizes $\mathbf{F}^\pi(s)$, which is proportional to the CDF value $\Phi_{Z^\pi}(\bar{z}(s,a))$. Supposing an agent need to select an actions between $a_1$ and $a_2$ to explore, Fig.~\ref{fig:two-action-example}(a) and (b) illustrate the CDF value (shaded area) for each action.

In these cases, the value distribution $Z^\pi(s, a)$ is specified  as Gaussian (Eq.~\ref{4-2-e15}), and the sampled optimistic value $\bar{z}(s,a)$ is specified  as the mean of OVD $\mu_{\bar{Z}}(s,a)$ (Eq.~\ref{4-2-e13}).
At state $s$, we assume that the means of $Z^\pi$ at actions $a_1$ and $a_2$ are the same for ease of clarification. 

\textit{Optimistic exploration:} Fig.~\ref{fig:two-action-example}(a) illustrates how OVD-Explorer achieves an optimistic exploration. 
Assuming the noise at $a_1$ and $a_2$ is equal, but epistemic uncertainty is higher at $a_1$, then $\mu_{\bar{Z}}(s,a_1)>\mu_{\bar{Z}}(s,a_2)$ and the CDF value is larger at $a_1$. Therefore, OVD-Explorer prefers $a_1$ with high epistemic uncertainty for an optimistic exploration. 

\begin{figure}[t]
\label{fig:CDF}
 \centering

    \subfigure []{
    \includegraphics[width=3.2cm]{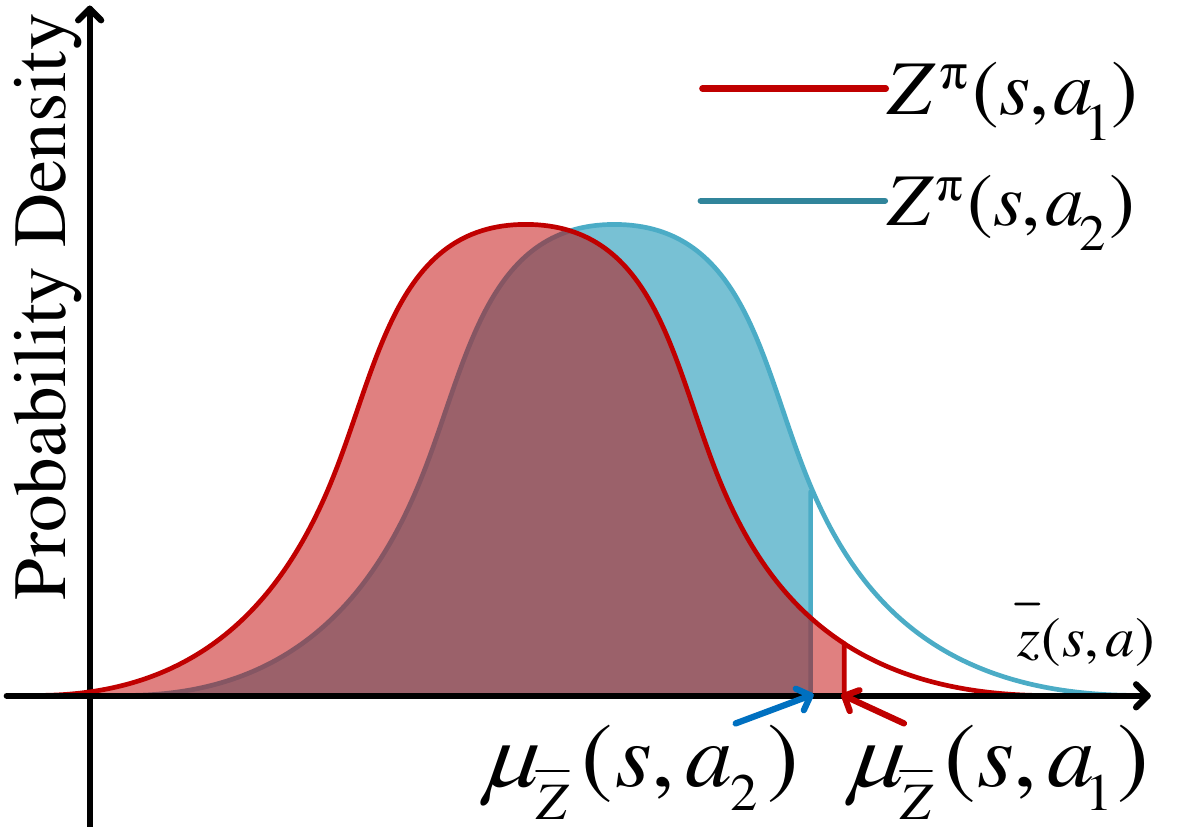}
    }
    \hspace{-1mm}
    \subfigure[]{
    \includegraphics[width=3.4cm]{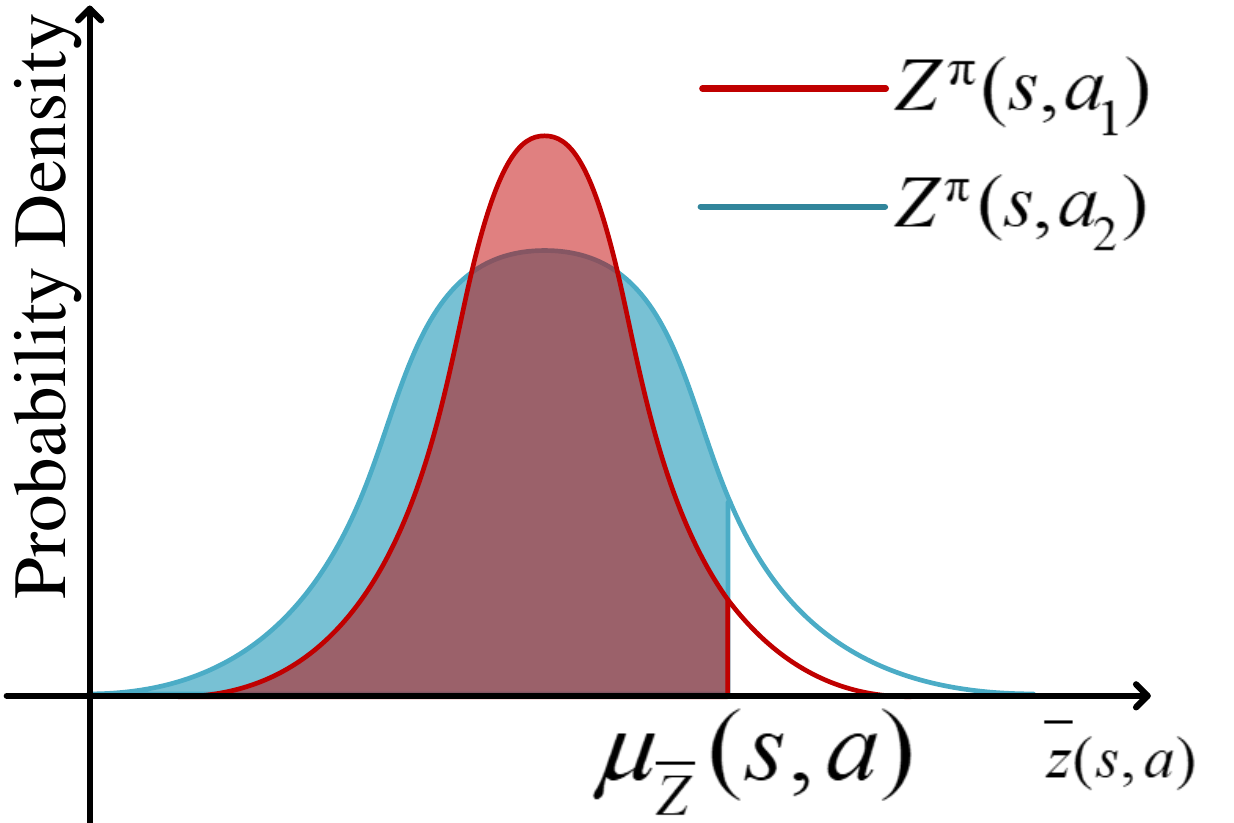}
    }
\caption{How OVD-Explorer explores (a) optimistically about epistemic uncertainty, (b) pessimistically about noise. 
}
\label{fig:two-action-example}
\end{figure}

\textit{Noise-aware exploration:} Fig.~\ref{fig:two-action-example}(b) demonstrates how OVD-Explorer behave noise-aware to avoid the area with higher noise (aleatoric uncertainty). When both actions have equal epistemic uncertainty, $\mu_{\bar{Z}}(s,a_1)$ = $\mu_{\bar{Z}}(s,a_2)$, and noise is lower at $a_1$ (PDF curve of $Z^\pi(s, a_1)$ is ``thinner and taller''), the CDF value will be larger at $a_1$. 
In such a case, OVD-Explorer prefers action $a_1$ with lower aleatoric uncertainty (i.e., lower noise) for a noise-aware exploration.

\textit{Adaptivity.} In the early training, the noise estimations of all actions are nearly identical (Fig.~\ref{fig:two-action-example}(a)), and the exploration is primarily guided by epistemic uncertainty. After sufficient training, the epistemic uncertainty decreases, while the noise estimation converges to true environment randomness (Fig.~\ref{fig:two-action-example}(b)). At this point, exploration strategy tends to be noise-avoiding.
OVD-Explorer seeks an adaptive balance of noise-aware optimistic exploration throughout the exploration process, which is a significant advantage compared to other OFU-based methods.

\section{OVD-Explorer for RL Algorithms}
\label{4-3}

For continuous RL, solving the argmax operator in Eq.~\ref{eq4-0-1} is intractable. In this section, aiming at maximizing ${\bf{F}}^\pi(s)$, we use a gradient-based approach to generate the behavior policy, and incorporate it with policy-based algorithms.

We denote the policy learned by any policy-based algorithm as $\pi_\phi$, parameterized by $\phi$.
To avoid the gap between $\pi_\phi$ and the training data collected by behavior policy $\pi_E$, we derive $\pi_E$ in the vicinity of $\pi_\phi$.
Then, aiming at maximizing ${\bf{F}}^{\pi_\phi}(s)$, we derive its gradient regarding the policy $\nabla_\phi   {\bf{F}}^{\pi_\phi}(s)$ using automatic differentiation and generate behavior policy $\pi_E$ by performing gradient ascent based on $\pi_\phi$. Thus $\pi_E$ can guide exploration towards maximizing the exploration ability continuously, performing noise-aware optimistic exploration.
Concretely, Proposition~\ref{prop-2} shows how to calculate $\pi_E$. 
\begin{proposition}
\label{prop-2}
Based on any policy $\pi_\phi=\mathcal{N}(\mu_\phi, \sigma_\phi)$, the OVD-Explorer behavior policy $\pi_E=\mathcal{N}(\mu_E, \Sigma_E)$ at given state $s$ is as follows:
\begin{equation}
\label{eq4-3-8}
 \mu_E=\mu_{\phi}+\alpha\mathbb{E}_{\bar{Z}^\pi}\left[m \times \frac{\partial \bar{z} (s, a)}{\partial a}|_{a=\mu_{\phi}}\right],
\end{equation}
and
\begin{equation}
\Sigma_E = \sigma_\phi.
\end{equation}
In specific, $m =  \log \frac{\Phi_{Z^\pi(s, \mu_{\phi})}(\bar{z} (s, \mu_{\phi}))}{C}+1$, $\bar{z}(s, a)$ is a sample from OVD $\bar{Z}^\pi$, and $\alpha$ controls the step size of the update along the gradient direction, representing the exploration degree~(see proof in App.~\ref{proof:prop}).
\end{proposition}

The expectation $\mathbb{E}_{\bar{Z}^\pi}$ can be estimated by K samples, then Eq.~\ref{eq4-3-8} is simplifies as:
\begin{equation}
\label{eq4-3-7}
    \mu_E=\mu_{\phi}+\frac{\alpha m}{K}\sum_{i=1}^{K} \frac{\partial \bar{z}_i (s, a)}{\partial a}|_{a=\mu_{\phi}}. 
\end{equation}
\begin{algorithm}[tb]
\caption{Behavior policy generation at step $t$.}
\label{alg1}
\textbf{Input}: Current state $ s_t $,  current value distribution estimators $\theta_1, \theta_2$, current policy network $\phi$.\\
\textbf{Output}: Behavior policy $\pi_E$.

\begin{algorithmic}[1]
\STATE Obtain policy $\pi_\phi(\cdot|s_t) \sim \mathcal{N}(\mu_\phi(s_t), \sigma_\phi(s_t))$
\STATE \textcolor{gray}{// Construct the distributions of return and upper bound}
\STATE Construct OVD $\bar{Z}^\pi(s_t,  \mu_\phi(s_t))$ using \text{Eq.~}\ref{4-2-z*}
\STATE Construct  $Z^{\pi}(s_t,  \mu_\phi(s_t)) $ using \text{Eq.~}\ref{4-2-e15}\text{ or }\ref{4-2-e17}
\STATE \textcolor{gray}{// Calculate the behavior policy}
\STATE Calculate the behavior policy's mean $\mu_E$ using \text{Eq.~}\ref{eq4-3-7}
\STATE \textbf{return} $\pi_E \sim \mathcal{N}(\mu_E, \sigma_\phi(s_t))$
\end{algorithmic}
\end{algorithm}

Algorithm~\ref{alg1} summarizes the procedure to generate a behavior policy at step $t$ of OVD-Explorer. 
Following Algorithm~\ref{alg1}, OVD-Explorer can be integrated with any existing policy-based RL algorithms from a distributional perspective, to render a stable and well-performed algorithm.
Specifically, given state $s_t$, based on the current policy (Line 1), by constructing the optimistic value distribution $\bar{Z}^\pi$ as well as the value distribution $Z^{\pi}$ of the policy (Line 3-4), the behavior policy derived from OVD-Explorer can be calculated directly using Proposition~\ref{prop-2} (Line 6-7).

\section{Experiments}
\label{sec:exp}

To reveal the consistency between our theoretical analysis and the performance of OVD-Explorer, and demonstrate the significant advantage over other advanced methods, we conduct experiments mainly for the following questions: \\
\textbf{\textit{RQ1 (Exploration ability)}}: Can OVD-Explorer explore as a noise-aware optimistic manner as expected? \\
\textbf{\textit{RQ2 (Performance)}}: Can  OVD-Explorer perform notable advantages on common continuous control benchmarks?\\
Due to space constraints, more experimental details and evaluation results can be found in the appendix.

\subsection{Baseline Algorithms and Implementation Details}

Our baseline algorithms include SAC~\cite{HaarnojaZAL18SAC}, DSAC~\cite{dsac}, and DOAC, an extension of the scalar Q-value  within the OAC~\cite{DBLP:conf/nips/CiosekVLH19} to distributional Q-value.
Our implementation of OVD-Explorer is based on the OAC repository, also refers to the code of DSAC \footnote{https://github.com/xtma/dsac} and softlearning \footnote{https://github.com/rail-berkeley/softlearning}. 
We implement OVD-Explorer\_G and OVD-Explorer\_Q (or abbreviated as OVDE\_G and OVDE\_Q), representing approaches to formulate the value distribution $Z^{\pi}$ using Eq.~\ref{4-2-e15} (\texttt{torch.distributions.Normal}) or Eq.~\ref{4-2-e17}, respectively. 
The key hyper-parameters associated with the exploration, i.e., the exploration ratio $\alpha$ and the uncertainty ratio $\beta$, are determined by grid search, with detailed information presented in the Appendix. Moreover, the hyperparameters related to the training procedure remain consistent across all algorithms.

All experiments are performed on NVIDIA GeForce RTX 2080 Ti 11GB graphics card. 
To counteract the randomness from a statistical perspective, we conduct multiple trials using different seeds. The final results of each trial are collected based on the mean undiscounted episodic return over the last 8\% epoch (or up to the last 100 epochs) to ensure impartiality and minimize bias.

\begin{figure}[t]
\centering
        \includegraphics[width=0.13\textwidth]{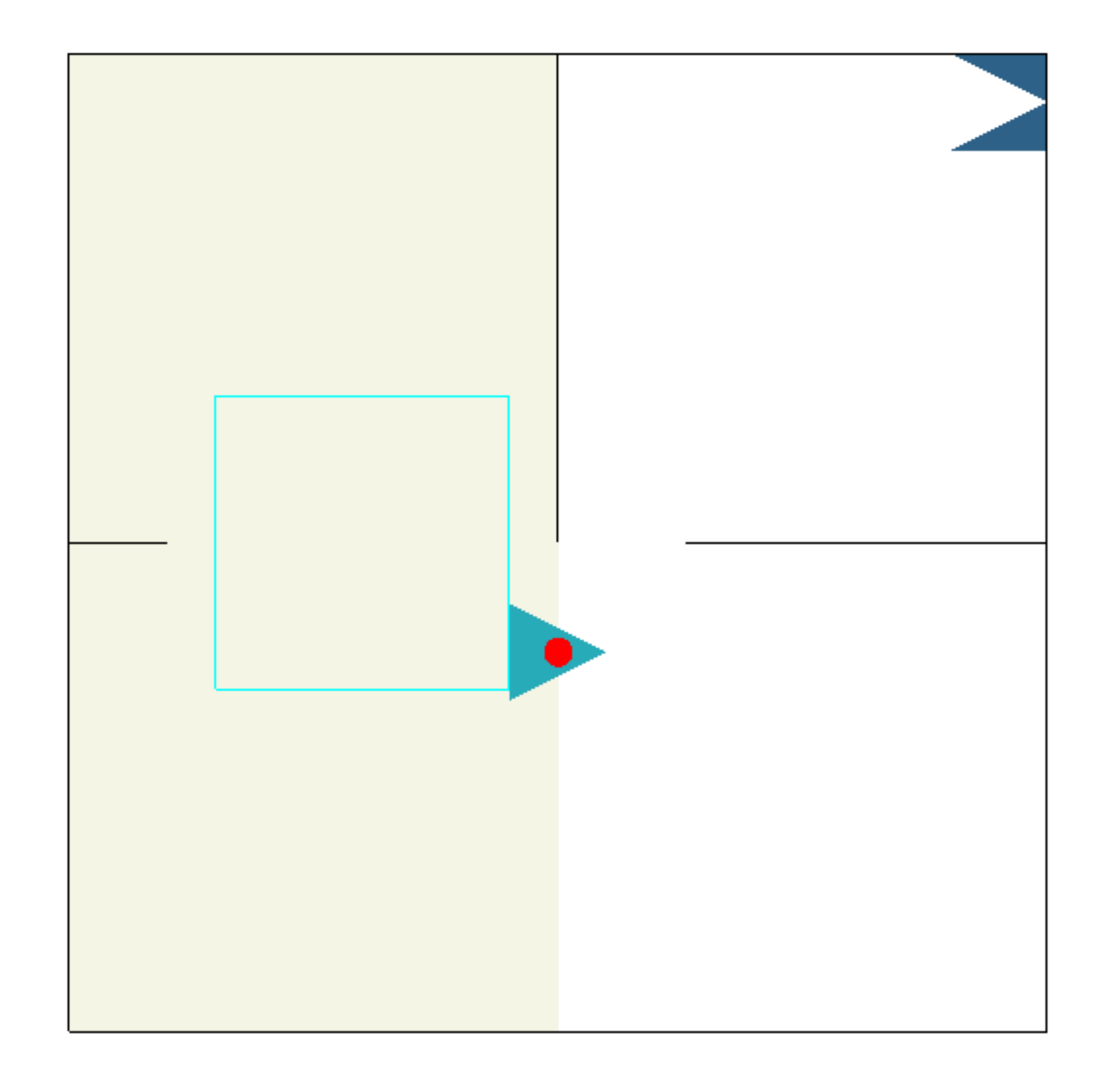}
        \includegraphics[width=0.31\textwidth]{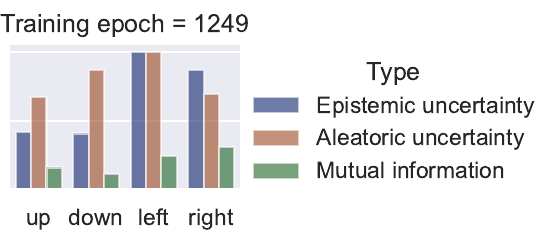}
      
    \caption{GridChaos. \textbf{Left}: In this map, the areas with darker background color have higher noise injected, and the agent aims at reaching the goal at the top right. \textbf{Right}: The values of uncertainty and exploration ability.}
    \label{fig:gridchaos_main}
\end{figure}

\subsection{Exploration in GridChaos (RQ1) }
\label{sec:rq1}
To illustrate that OVD-Explorer guides noise-aware optimistic exploration, we first evaluate OVD-Explorer on GridChaos. 
The GridChaos task is characterized by heterogeneous noise and sparse reward, making it particularly challenging and necessitating a robust capacity for noise-aware optimistic exploration to successfully accomplish the task.

\subsubsection{GridChaos Task}
GridChaos is built on OpenAI's Gym toolkit, as shown in Fig.~\ref{fig:gridchaos_main}(a). In GridChaos, the cyan triangle is under the agent's control, with the objective of reaching the fixed dark blue goal located at the top right corner.
The state is its current coordinate, and the action is a two-dimensional vector including the movement angle and distance. An episode terminates when the agent reaches the goal or maximum steps (typically 100). Also, it receives a +100 reward when reaching the goal, and otherwise 0. 
To simulate noise, heterogeneous Gaussian noise is injected into the state transitions.
Fig.~\ref{fig:gridchaos_main}(b) shows a basic performance comparison of algorithms, and show our advantage. More detailed performance are illustrated in the following.

\subsubsection{Exploration Patterns Analysis}
We first analyse the exploration pattern facilitated by OVD-Explorer.
We compare values of uncertainty and the exploration ability measurement~(in Proposition~\ref{the}) corresponding to distinct actions at the state in Fig.~\ref{fig:gridchaos_main}(a). 
Fig.~\ref{fig:gridchaos_main}(c) shows the values obtained at the 1249th training epoch. 
Basically, Fig.~\ref{fig:gridchaos_main}(c) shows that estimated aleatoric uncertainty (noise) of left is the highest, aligning with the environment's inherent attribute. This indicates that OVD-Explorer models the noise properly.
Further, OVD-explorer encourages to explore towards right, where the exploration ability (in green) is higher.
It implies that OVD-explorer balances the optimistic and noise in exploration, aligning with our intended objective.
Nevertheless, if the noise is not considered in exploration, the agent would be guided towards left, where the epistemic uncertainty is higher, then the agent may be trapped due to the high randomness in this area, potentially explaining why DOAC fails to effectively address such a stochastic task.

\subsubsection{Evaluation on Various Noise Scales}

To further empirically prove our strength, we test OVD-Explorer in GridChaos with various noise scales settings, as outlined in row \textit{A-E} in Tab.~\ref{tab:ablation-gridchaos1}. The column \textit{Noise setup} shows different scale of environmental heterogeneous Gaussian noise injected in four quadrants of Cartesian coordinate system. Note that the variance is not reported, as the mean values offer a comprehensive representation of results across 5 seeds. For instance, when the average result approaches 20, it indicates that only one seed successfully achieved the goal (obtaining a reward of +100) in the end. The row \textit{S} is the standard GridChaos as shown in Fig.~\ref{fig:gridchaos_main}(a).
Note that OVDE-m is specifically designed for the second set of experiment (see Sec.~\ref{sec:gridchaos_noise_goal}), thus we omit its results in this experiment. 
Remarkably, the results indicate that OVD-Explorer consistently achieves better performance across all the tested settings, particularly in scenarios with high levels of noise. This underscores OVD-Explorer's exploration capability in such noisy tasks, leading to more efficient learning and faster convergence towards the goal (see column \textit{FRG}).

Also, we make observations in the case without noise (row \textit{E}). Here, DSAC fails in any run across 5 seeds, while DOAC reachs the goal in one run but at a slower pace. OVD-Explorer achieves the goal swiftly in one run. This highlight the capability of OVD-Explorer and the highly challenging nature of the task, emphasizing the significance of employing a robust exploration strategy. 

\begin{table}[t]
\tiny
\begin{center}
\begin{tabular}{>{\centering}p{0.01cm}>{\raggedright}p{0.2cm}>{\raggedright}p{0.2cm}>{\raggedright}p{0.2cm}>{\raggedright}p{0.2cm}>{\raggedright}p{0.4cm}>{\raggedright}p{0.4cm}>{\raggedright}p{0.4cm}>{\raggedright}p{0.25cm}>{\raggedright}p{0.25cm}>{\raggedright}p{0.4cm}}
\toprule 
 & \multicolumn{4}{c}{{\tiny{}Noise setup}} & \multicolumn{3}{c}{{\tiny{}Average return}} & \multicolumn{3}{c}{{\tiny{}FRG epoch}}\tabularnewline
 & {\tiny{}1 } & {\tiny{}2 } & {\tiny{}3 } & {\tiny{}4 } & {\tiny{}DSAC } & {\tiny{}OVDE } & {\tiny{}DOAC } & {\tiny{}DSAC} & {\tiny{}OVDE} & {\tiny{}DOAC}\tabularnewline
\midrule 
{\tiny{}s} & {\tiny{}0.1 } & {\tiny{}0.5 } & {\tiny{}0.5 } & {\tiny{}0.1 } & {\tiny{}0.00} & \textbf{\tiny{}59.30}  & {\tiny{}3.02} & {\tiny{}1250+ } & \textbf{\tiny{}229}{\tiny{} }& {\tiny{}1222 }\tabularnewline
{\tiny{}a} & {\tiny{}0.0 } & {\tiny{}0.5 } & {\tiny{}0.1 } & {\tiny{}0.1 } & {\tiny{}18.94} & \textbf{\tiny{}58.99}  & {\tiny{}38.42 } & {\tiny{}1161 } & \textbf{\tiny{}180}{\tiny{} }  & {\tiny{}662 }\tabularnewline
{\tiny{}b} & {\tiny{}0.0 } & {\tiny{}0.05 } & {\tiny{}0.01 } & {\tiny{}0.01 } & {\tiny{}39.78} & \textbf{\tiny{}79.52}  & {\tiny{}18.71} & {\tiny{}694 } & \textbf{\tiny{}144}{\tiny{} } & {\tiny{}846 }\tabularnewline
{\tiny{}c} & {\tiny{}0.05 } & {\tiny{}0.1 } & {\tiny{}0.1 } & {\tiny{}0.05 } & {\tiny{}0.05} & \textbf{\tiny{}39.64} & {\tiny{}20.59} & {\tiny{}1250+ } & \textbf{\tiny{}180}{\tiny{} } & {\tiny{}309 }\tabularnewline
{\tiny{}d} & {\tiny{}0.001} & {\tiny{}0.005} & {\tiny{}0.005} & {\tiny{}0.001} & {\tiny{}20.00} & \textbf{\tiny{}40.46}  & {\tiny{}39.99} & {\tiny{}\textbf{284} } & \textbf{\tiny{}276}{\tiny{} }  & {\tiny{}\textbf{321} }\tabularnewline
{\tiny{}e} & {\tiny{}0.0 } & {\tiny{}0.0 } & {\tiny{}0.0 } & {\tiny{}0.0 } & {\tiny{}0.00} & \textbf{\tiny{}20.20}  & {\tiny{}14.60 } & {\tiny{}1250+ } & \textbf{\tiny{}185}{\tiny{} }& {\tiny{}1118 }\tabularnewline

\bottomrule
\end{tabular}
\caption{The averaged return of 5 runs for GridChaos (the first part). FRG epoch means the minimum training epochs to \textit{Firstly Reach the Goal} before totally 1250 epochs.}
\label{tab:ablation-gridchaos1}
\end{center}
\end{table}

\begin{table}[t]
\tiny
\begin{center}
\begin{tabular}{>{\centering}p{0.01cm}>{\raggedright}p{0.18cm}>{\raggedright}p{0.18cm}>{\raggedright}p{0.18cm}>{\raggedright}p{0.18cm}>{\raggedright}p{0.38cm}>{\raggedright}p{0.38cm}>{\raggedright}p{0.55cm}>{\raggedright}p{0.38cm}>{\raggedright}p{0.25cm}>{\raggedright}p{0.25cm}>{\raggedright}p{0.4cm}}
\toprule 
 & \multicolumn{4}{c}{{\tiny{}Noise setup}} & \multicolumn{4}{c}{{\tiny{}Average return}} & \multicolumn{3}{c}{{\tiny{}FRG epoch}}\tabularnewline
 & {\tiny{}1 } & {\tiny{}2 } & {\tiny{}3 } & {\tiny{}4 } & {\tiny{}DSAC } & {\tiny{}OVDE } & {\tiny{}OVDE(m) } & {\tiny{}DOAC } & {\tiny{}DSAC} & {\tiny{}OVDE} & {\tiny{}DOAC}\tabularnewline
\midrule 

{\tiny{}f} & {\tiny{}0.1 } & {\tiny{}0.05 } & {\tiny{}0.05 } & {\tiny{}0.1 } & {\tiny{}0.00} & {\tiny{}19.84} & \textbf{\tiny{}39.96}& {\tiny{}20.14} & {\tiny{}1250+ } & \textbf{\tiny{}188}{\tiny{} }  & {\tiny{}\textbf{233} }\tabularnewline
{\tiny{}g} & {\tiny{}0.05 } & {\tiny{}0.005} & {\tiny{}0.005} & {\tiny{}0.05 } & {\tiny{}0.00} & \textbf{\tiny{}20.69} & \textbf{\tiny{}20.07}{\tiny{}} & {\tiny{}0.00} & {\tiny{}1250+ } & \textbf{\tiny{}247}{\tiny{} } & {\tiny{}1250+ }\tabularnewline
{\tiny{}h} & {\tiny{}0.01 } & {\tiny{}0.005} & {\tiny{}0.005} & {\tiny{}0.01 } & {\tiny{}0.0 } & {\tiny{}40.00} & \textbf{\tiny{}60.00}& {\tiny{}39.99 } & {\tiny{}1250+ } & \textbf{\tiny{}200}{\tiny{} } & {\tiny{}{301} }\tabularnewline
{\tiny{}i} & {\tiny{}0.005 } & {\tiny{}0.001} & {\tiny{}0.001} & {\tiny{}0.005 } & {\tiny{}20.00 } & \textbf{\tiny{}39.98} & {\tiny{}20.00} & {\tiny{}20.00} & {\tiny{}\textbf{236} } & {\tiny{}\textbf{312} }  & {\tiny{}\textbf{296} }\tabularnewline
\bottomrule
\end{tabular}
\caption{The averaged return of 5 runs (the second part).}
\label{tab:ablation-gridchaos2}
\end{center}
\end{table}

\begin{table*}[t]
\scriptsize
\begin{center}
\renewcommand{\arraystretch}{0.825}
\begin{tabular}{lr@{\extracolsep{0pt}.}lr@{\extracolsep{0pt}.}lr@{\extracolsep{0pt}.}lr@{\extracolsep{0pt}.}lr@{\extracolsep{0pt}.}lr@{\extracolsep{0pt}.}l}
\toprule 
\textbf{Task}  & \multicolumn{2}{c}{\textbf{Epoch} } & \multicolumn{2}{c}{\textbf{SAC} } & \multicolumn{2}{c}{\textbf{DSAC }} & \multicolumn{2}{c}{\textbf{DOAC} } & \multicolumn{2}{c}{\textbf{OVD-Explorer\_G} } & \multicolumn{2}{c}\textbf{{OVD-Explorer\_Q} }\tabularnewline
\midrule 
Ant-v2  & \multicolumn{2}{c}{2500 } & 4867&8$\pm$1658.7  & 6385&9$\pm$1287.2  & 6625&4$\pm$746.8  & 7175&3$\pm$789.0  & \textbf{7382}&\textbf{3}$\pm$466.6 \tabularnewline
HalfCheetah-v2  & \multicolumn{2}{c}{2500 } & 11619&8$\pm$1642.01  & 13348&4$\pm$1957.1  & 12987&6$\pm$148.1  & 14796&2$\pm$1473.2  & \textbf{16484}&\textbf{3}$\pm$1373.75 \tabularnewline
Hopper-v2  & \multicolumn{2}{c}{1250 } & \textbf{2593}&\textbf{5}$\pm$574.7  & 2506&0$\pm$390.56  & 2353&0$\pm$754.1  & 2394&6$\pm$496.6  & 2559&3$\pm$384.5 \tabularnewline
Reacher-v2  & \multicolumn{2}{c}{250 } & -22&7$\pm$2.0  & -12&2$\pm$1.6  & -18&7$\pm$1.7  & -11&6$\pm$1.0  & \textbf{-11}&\textbf{3}$\pm$1.2 \tabularnewline
InvDbPendulum-v2  & \multicolumn{2}{c}{300 } & 9306&2$\pm$89.5 & 8916&9$\pm$1041.7  & 5798&5$\pm$3439.0  & 9263&8$\pm$189.1  & \textbf{9355}&\textbf{0}$\pm$12.1 \tabularnewline
\midrule
N-Ant-v2  & \multicolumn{2}{c}{2500 } & 222&96$\pm$41.93  & 465&34$\pm$53.94  & 344&71$\pm$20.39  & \textbf{524}&\textbf{16}$\pm$10.54  & 513&77$\pm$17.87 \tabularnewline
N-HalfCheetah-v2  & \multicolumn{2}{c}{1250 } & 368&57$\pm$28.01  & 431&81$\pm$39.41  & 402&26$\pm$37.27  & 447&3$\pm$38.57  & \textbf{453}&\textbf{56}$\pm$55.97 \tabularnewline
N-Hopper-v2  & \multicolumn{2}{c}{1250 } & 213&71$\pm$21.97  & {238}&{62}$\pm$19.89  & \textbf{252}&\textbf{53}$\pm$13.07  & 234&88$\pm$15.24  & {239}&{43}$\pm$9.90 \tabularnewline
N-Pusher-v2  & \multicolumn{2}{c}{1250 } & -50&57$\pm$20.65  & -27&33$\pm$3.79  & -29&82$\pm$4.29  & \textbf{-25}&\textbf{69}$\pm$3.57  & -26&13$\pm$3.63 \tabularnewline
N-InvDbPendulum-v2  & \multicolumn{2}{c}{300 } & 931&63$\pm$7.14  & \textbf{932}&\textbf{81}$\pm$1.87  & 381&87$\pm$139.36  & \textbf{932}&\textbf{70}$\pm$2.26  & \textbf{933}&\textbf{54}$\pm$2.69 \tabularnewline
\midrule
\textbf{Average} (standard tasks) & \multicolumn{2}{c}{} & \multicolumn{2}{c}{5672.92} & \multicolumn{2}{c}{6229.00}	& \multicolumn{2}{c}{5549.16}	& \multicolumn{2}{c}{6723.66}	& \multicolumn{2}{c}{\textbf{7153.92}}  \tabularnewline
\textbf{Average} (noisy tasks) & \multicolumn{2}{c}{} &  \multicolumn{2}{c}{337.26} &	\multicolumn{2}{c}{408.25} &	\multicolumn{2}{c}{270.31} &	\multicolumn{2}{c}{\textbf{422.67}}	& \multicolumn{2}{c}{\textbf{422.83}} \tabularnewline
\bottomrule
\end{tabular}
\caption{Comparisons of algorithms on five standard and five noisy tasks. The averaged performance and standard deviation of 10 runs are reported. The training epoch count is shown in column \textit{epoch}, and the best values of each row are shown in bold.}
\label{tab:performance-mujoco}
\end{center}
\end{table*}

\subsubsection{Evaluation on Tasks in Which the Noise is High around the Goal}
\label{sec:gridchaos_noise_goal}
To verify whether the noise avoidance ability of OVD-Explorer dominates the exploration process when the noise around the target is higher, we conduct the experiment where the noise in the right half (where the goal is located), is set larger.
The results are shown in row \textit{F-G} in Tab.~\ref{tab:ablation-gridchaos2}. Note that we use OVDE to denote the usual implementation that pessimistically estimates the value distribution (i.e., using Eq.~\ref{4-2-e15}). Besides, OVDE(m) denotes the implementation that does not pessimistically estimate the value distribution (i.e., we modify the mean of Gaussian distribution $Z^\pi$ in Eq.~\ref{4-2-e15} from the lower bound to expected value of the $Q$ estimation $\mu(s, a)$ as in Eq.~\ref{4-2-e13}). 

Overall, in most cases, OVD-Explorer guides better exploration and perform better than baselines. This highlights OVD-Explorer's ability to handle various scenarios effectively, even in tasks with higher noise levels around the goal. 

Moreover, an intriguing observation is that OVD-Explorer may exhibit better performance when the pessimistic estimation is turned off in the presence of higher noise around the goal (see column OVDE(m)). This finding suggests that excessive pessimism may not be necessary when there is a crucial need to explore areas characterized by high aleatoric uncertainty. In such cases, a more balanced approach may lead to improved results.

\begin{figure}[t]
\centering

\subfigure[250 steps]{
\includegraphics[width=4.3cm]{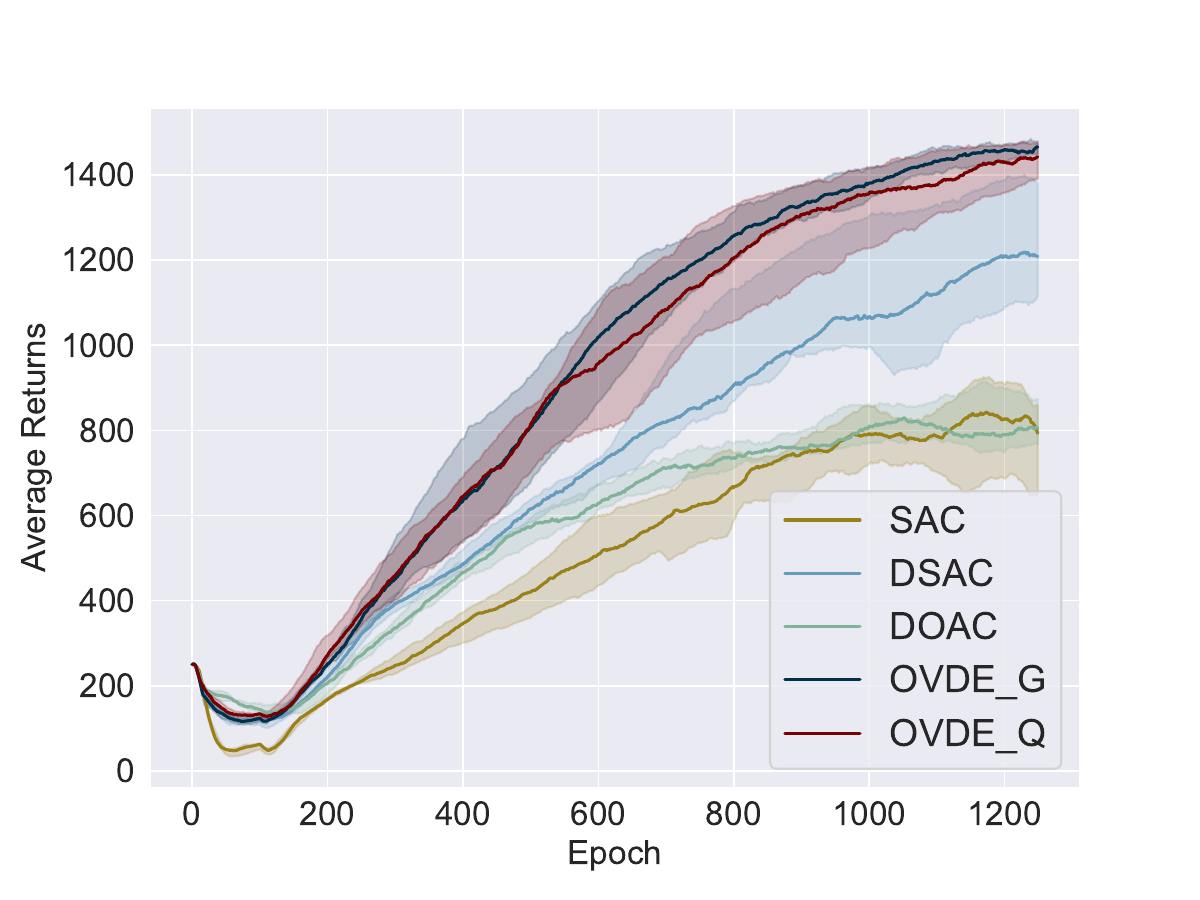}
}
\hspace{-8mm}
\subfigure[500 steps]{
\includegraphics[width=4.3cm]{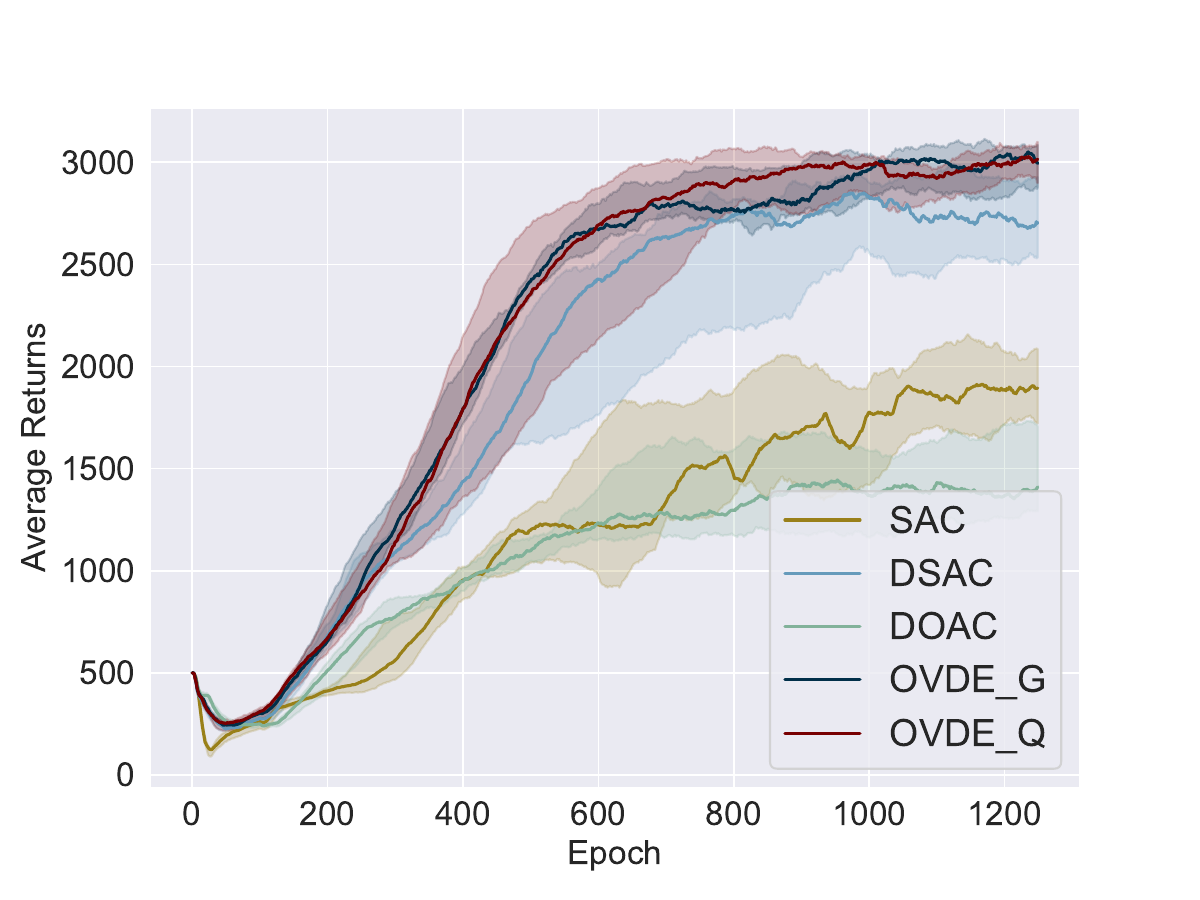}
}

\caption{Training curves on Noisy Ant-v2 tasks with different maximum episodic length setup. The sub-title of each figure represents the episodic horizon. We report the median of returns and the interquartile range of 10 runs. Curves are smoothed uniformly for visual clarity.}
\label{fig:ablation-horizon-1}
\end{figure}

\subsection{Performance on Mujoco Tasks (RQ2)}
\label{sec:rq2}

To showcase the broader efficacy of OVD-Explorer, we conduct experiments encompassing 5 standard and 5 stochastic tasks based on wildly-used continuous control benchmark, Gym Mujoco. Tab.~\ref{tab:performance-mujoco} shows the averaged performance and standard deviation of 10 seeds. 
It is important to note that for the standard tasks\footnote{\url{https://github.com/openai/gym/tree/master/gym/envs/mujoco}}, the dynamics are deterministic, and any observed noise is ascribed to the stochastic policy employed. Conversely, in the case of the five noisy tasks (indicated by the prefix N-), Gaussian noise of varying scales is randomly injected into each state transition.  The investigation yields compelling insights into the capabilities of OVD-Explorer. 
\par

Primarily, OVD-Explorer can perform stably in standard task. The results reveal that DSAC consistently outperforms SAC, underscoring the advantage that value distribution brings to policy evaluation. 
Notably, OVD-Explorer surpasses baseline algorithms significantly, particularly in high-dimensional tasks, such as Ant-v2 and HalfCheetah-v2. 
These evaluations on standard tasks convincingly demonstrate OVD-Explorer's remarkable proficiency in promoting optimistic exploration, highlighting its universal efficacy in exploration capabilities.

Second, these results underscore the efficacy of OVD-Explorer in exploring noisy environments while avoiding the adverse impact of noise. This is exemplified by experiments on 5 noisy tasks (Tab.~\ref{tab:performance-mujoco}) and Noisy Ant-v2 tasks with varying maximum episodic lengths (Fig.~\ref{fig:ablation-horizon}, detailed results in the appendix).
Notably, DOAC's performance is even inferior to DSAC in most tasks, indicating that the presence of heteroscedastic noise significantly interferes with the exploration process guided by DOAC. In contrast, OVD-Explorer exhibits substantial advantages over SAC and DOAC, notably outperforming DSAC in most cases.

Furthermore, regarding the two implementations (namely OVDE\_G and OVDE\_Q) of OVD-Explorer, we observe that OVDE\_Q consistently demonstrates greater stability. 
The key distinction between these implementations lies in the formulations of $Z^\pi(s, a)$. OVDE\_Q's employment of quantile distribution offers higher flexibility, allowing for a more accurate characterization of the value distribution. Conversely, OVDE\_ G, reliant on the Gaussian prior, exhibits limited capacity in this regard, leading to a relatively diminished performance in some cases.

\section{Conclusion}
\label{sec:con}
In this paper, we have presented OVD-Explorer, a novel noise-aware optimistic exploration method for continuous RL. By introducing a unique measurement of exploration ability and maximizing it, OVD-Explorer effectively generates a behavior policy that adheres to the OFU principle. Also, it intelligently avoids excessive exploration in areas with high noise, thereby mitigating the adverse effects of noise. Consistently across tasks with no noise as well as various forms of noise, the experiment underscores our performance advantages.
Moving forward, we recognize the potential for extending OVD-Explorer to discrete tasks and even Multi-agent tasks. This will enhance the versatility of OVD-Explorer, affording it the capability to effectively confront the challenges of exploration in noisy environments that are widespread across diverse real-world scenarios.

\section{Acknowledgments}
This work is supported by the National Key R\&D Program of China (Grant No. 2022ZD0116402), the National Natural Science Foundation of China (Grant Nos. 92370132, 62106172).

\bibliography{main}

\begin{thebibliography}{51}
\providecommand{\natexlab}[1]{#1}

\bibitem[{Antos, Munos, and Szepesv{\'{a}}ri(2007)}]{AntosMS07}
Antos, A.; Munos, R.; and Szepesv{\'{a}}ri, C. 2007.
\newblock Fitted Q-iteration in continuous action-space MDPs.
\newblock In \emph{Advances in NIPS, Vancouver, British Columbia, Canada, December 3-6, 2007}, 9--16. Curran Associates, Inc.

\bibitem[{Auer, Cesa{-}Bianchi, and Fischer(2002)}]{AuerCF02}
Auer, P.; Cesa{-}Bianchi, N.; and Fischer, P. 2002.
\newblock Finite-time Analysis of the Multiarmed Bandit Problem.
\newblock \emph{Mach. Learn.}, 47(2-3): 235--256.

\bibitem[{Badia et~al.(2020)Badia, Sprechmann, Vitvitskyi, Guo, Piot, Kapturowski, Tieleman, Arjovsky, Pritzel, Bolt, and Blundell}]{Badia20NGU}
Badia, A.~P.; Sprechmann, P.; Vitvitskyi, A.; Guo, D.; Piot, B.; Kapturowski, S.; Tieleman, O.; Arjovsky, M.; Pritzel, A.; Bolt, A.; and Blundell, C. 2020.
\newblock Never Give Up: Learning Directed Exploration Strategies.
\newblock \emph{arXiv preprint arXiv:2002.06038}.

\bibitem[{Bai et~al.(2021{\natexlab{a}})Bai, Wang, Han, Garg, Hao, Liu, and Wang}]{bai-DB}
Bai, C.; Wang, L.; Han, L.; Garg, A.; Hao, J.; Liu, P.; and Wang, Z. 2021{\natexlab{a}}.
\newblock Dynamic bottleneck for robust self-supervised exploration.
\newblock In \emph{Advances in Neural Information Processing Systems}, volume~34, 17007--17020.

\bibitem[{Bai et~al.(2021{\natexlab{b}})Bai, Wang, Han, Hao, Garg, Liu, and Wang}]{bai-OB2I}
Bai, C.; Wang, L.; Han, L.; Hao, J.; Garg, A.; Liu, P.; and Wang, Z. 2021{\natexlab{b}}.
\newblock Principled Exploration via Optimistic Bootstrapping and Backward Induction.
\newblock In \emph{International Conference on Machine Learning}, volume 139, 577--587.

\bibitem[{Bai et~al.(2022)Bai, Xiao, Zhu, Wang, Zhou, Garg, He, Liu, and Wang}]{bai-MQN}
Bai, C.; Xiao, T.; Zhu, Z.; Wang, L.; Zhou, F.; Garg, A.; He, B.; Liu, P.; and Wang, Z. 2022.
\newblock Monotonic quantile network for worst-case offline reinforcement learning.
\newblock \emph{IEEE Transactions on Neural Networks and Learning Systems}.

\bibitem[{Belakaria, Deshwal, and Doppa(2020)}]{DBLP:journals/corr/abs-2009-01721}
Belakaria, S.; Deshwal, A.; and Doppa, J.~R. 2020.
\newblock Max-value Entropy Search for Multi-Objective Bayesian Optimization with Constraints.
\newblock \emph{CoRR}, abs/2009.01721.

\bibitem[{Belghazi et~al.(2018)Belghazi, Baratin, Rajeswar, Ozair, Bengio, Hjelm, and Courville}]{BelghaziBROBHC18}
Belghazi, M.~I.; Baratin, A.; Rajeswar, S.; Ozair, S.; Bengio, Y.; Hjelm, R.~D.; and Courville, A.~C. 2018.
\newblock Mutual Information Neural Estimation.
\newblock In \emph{ICML 2018, Stockholmsm{\"{a}}ssan, Stockholm, Sweden}, volume~80, 530--539. {PMLR}.

\bibitem[{Bellemare, Dabney, and Munos(2017)}]{BellemareDM17}
Bellemare, M.~G.; Dabney, W.; and Munos, R. 2017.
\newblock A Distributional Perspective on Reinforcement Learning.
\newblock In \emph{Proceedings of the 34th ICML, {ICML} 2017, Sydney, NSW, Australia, 6-11 August 2017}, volume~70 of \emph{Proceedings of Machine Learning Research}, 449--458. {PMLR}.

\bibitem[{Bellemare et~al.(2016)Bellemare, Srinivasan, Ostrovski, Schaul, Saxton, and Munos}]{BellemareSOSSM16}
Bellemare, M.~G.; Srinivasan, S.; Ostrovski, G.; Schaul, T.; Saxton, D.; and Munos, R. 2016.
\newblock Unifying Count-Based Exploration and Intrinsic Motivation.
\newblock In \emph{Advances in NIPS}, 1471--1479.

\bibitem[{Chen et~al.(2017)Chen, Sidor, Abbeel, and Schulman}]{chen2017ucb}
Chen, R.~Y.; Sidor, S.; Abbeel, P.; and Schulman, J. 2017.
\newblock UCB exploration via Q-ensembles.
\newblock \emph{arXiv preprint arXiv:1706.01502}.

\bibitem[{Cheng et~al.(2020)Cheng, Hao, Dai, Liu, Gan, and Carin}]{ChengHDLGC20}
Cheng, P.; Hao, W.; Dai, S.; Liu, J.; Gan, Z.; and Carin, L. 2020.
\newblock {CLUB:} {A} Contrastive Log-ratio Upper Bound of Mutual Information.
\newblock In \emph{ICML 2020, 13-18 July 2020, Virtual Event}, volume 119 of \emph{Proceedings of Machine Learning Research}, 1779--1788. {PMLR}.

\bibitem[{Ciosek et~al.(2019)Ciosek, Vuong, Loftin, and Hofmann}]{DBLP:conf/nips/CiosekVLH19}
Ciosek, K.; Vuong, Q.; Loftin, R.; and Hofmann, K. 2019.
\newblock Better Exploration with Optimistic Actor Critic.
\newblock In \emph{Advances in NeurIPS 2019, 8-14 December 2019, Vancouver, BC, Canada}, 1785--1796.

\bibitem[{Clements et~al.(2019)Clements, Robaglia, Delft, Slaoui, and Toth}]{abs-1905-09638}
Clements, W.~R.; Robaglia, B.; Delft, B.~V.; Slaoui, R.~B.; and Toth, S. 2019.
\newblock Estimating Risk and Uncertainty in Deep Reinforcement Learning.
\newblock \emph{CoRR}, abs/1905.09638.

\bibitem[{Dabney et~al.(2018{\natexlab{a}})Dabney, Ostrovski, Silver, and Munos}]{DabneyOSM18}
Dabney, W.; Ostrovski, G.; Silver, D.; and Munos, R. 2018{\natexlab{a}}.
\newblock Implicit Quantile Networks for Distributional Reinforcement Learning.
\newblock In \emph{Proceedings of the 35th ICML, {ICML} 2018, Stockholmsm{\"{a}}ssan, Stockholm, Sweden, July 10-15, 2018}, volume~80 of \emph{Proceedings of Machine Learning Research}, 1104--1113. {PMLR}.

\bibitem[{Dabney et~al.(2018{\natexlab{b}})Dabney, Rowland, Bellemare, and Munos}]{DabneyRBM18QRDQN}
Dabney, W.; Rowland, M.; Bellemare, M.~G.; and Munos, R. 2018{\natexlab{b}}.
\newblock Distributional Reinforcement Learning With Quantile Regression.
\newblock In \emph{AAAI, New Orleans, Louisiana, USA, February 2-7, 2018}, 2892--2901. {AAAI} Press.

\bibitem[{Dalal et~al.(2018)Dalal, Dvijotham, Vecer{\'{\i}}k, Hester, Paduraru, and Tassa}]{safeexplorationc20118}
Dalal, G.; Dvijotham, K.; Vecer{\'{\i}}k, M.; Hester, T.; Paduraru, C.; and Tassa, Y. 2018.
\newblock Safe Exploration in Continuous Action Spaces.
\newblock \emph{CoRR}, abs/1801.08757.

\bibitem[{Ding et~al.(2021)Ding, Wei, Yang, Wang, and Jovanovic}]{DingWYWJ21safeexp}
Ding, D.; Wei, X.; Yang, Z.; Wang, Z.; and Jovanovic, M.~R. 2021.
\newblock Provably Efficient Safe Exploration via Primal-Dual Policy Optimization.
\newblock In \emph{AISTATS 2021, April 13-15, 2021, Virtual Event}, volume 130 of \emph{Proceedings of Machine Learning Research}, 3304--3312. {PMLR}.

\bibitem[{Fujimoto, van Hoof, and Meger(2018)}]{FujimotoHM18}
Fujimoto, S.; van Hoof, H.; and Meger, D. 2018.
\newblock Addressing Function Approximation Error in Actor-Critic Methods.
\newblock In \emph{Proceedings of the 35th ICML, {ICML} 2018, Stockholmsm{\"{a}}ssan, Stockholm, Sweden, July 10-15, 2018}, volume~80 of \emph{Proceedings of Machine Learning Research}, 1582--1591. {PMLR}.

\bibitem[{Gal and Ghahramani(2016)}]{GalG16dropout}
Gal, Y.; and Ghahramani, Z. 2016.
\newblock Dropout as a Bayesian Approximation: Representing Model Uncertainty in Deep Learning.
\newblock In \emph{Proceedings of the 33nd ICML, {ICML} 2016, New York City, NY, USA, June 19-24, 2016}, volume~48 of \emph{{JMLR} Workshop and Conference Proceedings}, 1050--1059. JMLR.org.

\bibitem[{Greenberg et~al.(2022)Greenberg, Chow, Ghavamzadeh, and Mannor}]{efficient_risk_averse}
Greenberg, I.; Chow, Y.; Ghavamzadeh, M.; and Mannor, S. 2022.
\newblock Efficient Risk-Averse Reinforcement Learning.

\bibitem[{Haarnoja et~al.(2018)Haarnoja, Zhou, Abbeel, and Levine}]{HaarnojaZAL18SAC}
Haarnoja, T.; Zhou, A.; Abbeel, P.; and Levine, S. 2018.
\newblock Soft Actor-Critic: Off-Policy Maximum Entropy Deep Reinforcement Learning with a Stochastic Actor.
\newblock In \emph{{ICML} 2018, Stockholmsm{\"{a}}ssan, Stockholm, Sweden, July 10-15, 2018}, volume~80 of \emph{Proceedings of Machine Learning Research}, 1856--1865. {PMLR}.

\bibitem[{Hao et~al.(2023)Hao, Yang, Tang, Bai, Liu, Meng, Liu, and Wang}]{hao2023exploration}
Hao, J.; Yang, T.; Tang, H.; Bai, C.; Liu, J.; Meng, Z.; Liu, P.; and Wang, Z. 2023.
\newblock Exploration in deep reinforcement learning: From single-agent to multiagent domain.
\newblock \emph{IEEE Transactions on Neural Networks and Learning Systems}.

\bibitem[{Hinton and van Camp(1993)}]{DBLP:conf/colt/HintonC93}
Hinton, G.~E.; and van Camp, D. 1993.
\newblock Keeping the Neural Networks Simple by Minimizing the Description Length of the Weights.
\newblock In \emph{Proceedings of the Sixth Annual {ACM} Conference on Computational Learning Theory, {COLT} 1993, Santa Cruz, CA, USA, July 26-28, 1993}, 5--13. {ACM}.

\bibitem[{Houthooft et~al.(2016)Houthooft, Chen, Duan, Schulman, De~Turck, and Abbeel}]{houthooft2016vime}
Houthooft, R.; Chen, X.; Duan, Y.; Schulman, J.; De~Turck, F.; and Abbeel, P. 2016.
\newblock Vime: Variational information maximizing exploration.
\newblock In \emph{Advances in NIPS}, 1109--1117.

\bibitem[{Keramati et~al.(2020)Keramati, Dann, Tamkin, and Brunskill}]{KeramatiDTB20risk-averse}
Keramati, R.; Dann, C.; Tamkin, A.; and Brunskill, E. 2020.
\newblock Being Optimistic to Be Conservative: Quickly Learning a CVaR Policy.
\newblock In \emph{AAAI 2020, New York, NY, USA, February 7-12, 2020}, 4436--4443. {AAAI} Press.

\bibitem[{Kim et~al.(2019)Kim, Kim, Jeong, Levine, and Song}]{KimKJLS19EMI}
Kim, H.; Kim, J.; Jeong, Y.; Levine, S.; and Song, H.~O. 2019.
\newblock {EMI:} Exploration with Mutual Information.
\newblock In \emph{ICML 2019, 9-15 June 2019, Long Beach, California, {USA}}, volume~97 of \emph{Proceedings of Machine Learning Research}, 3360--3369.

\bibitem[{Kingma and Ba(2015)}]{KingmaB14}
Kingma, D.~P.; and Ba, J. 2015.
\newblock Adam: {A} Method for Stochastic Optimization.
\newblock In \emph{{ICLR} 2015, San Diego, CA, USA, May 7-9, 2015, Conference Track Proceedings}.

\bibitem[{Kirschner and Krause(2018)}]{KirschnerK18}
Kirschner, J.; and Krause, A. 2018.
\newblock Information Directed Sampling and Bandits with Heteroscedastic Noise.
\newblock In \emph{Conference On Learning Theory, {COLT} 2018, Stockholm, Sweden, 6-9 July 2018}, volume~75 of \emph{Proceedings of Machine Learning Research}, 358--384. {PMLR}.

\bibitem[{Koenker and Hallock(2001)}]{koenker2001quantile}
Koenker, R.; and Hallock, K.~F. 2001.
\newblock Quantile regression.
\newblock \emph{Journal of economic perspectives}, 15(4): 143--156.

\bibitem[{Lee et~al.(2021)Lee, Laskin, Srinivas, and Abbeel}]{LeeLSA21}
Lee, K.; Laskin, M.; Srinivas, A.; and Abbeel, P. 2021.
\newblock {SUNRISE:} {A} Simple Unified Framework for Ensemble Learning in Deep Reinforcement Learning.
\newblock In \emph{ICML 2021, 18-24 July 2021, Virtual Event}, volume 139 of \emph{Proceedings of Machine Learning Research}, 6131--6141. {PMLR}.

\bibitem[{Li et~al.(2021)Li, Tang, Zheng, Hao, Li, Wang, Meng, and Wang}]{hyar}
Li, B.; Tang, H.; Zheng, Y.; Hao, J.; Li, P.; Wang, Z.; Meng, Z.; and Wang, L. 2021.
\newblock HyAR: Addressing Discrete-Continuous Action Reinforcement Learning via Hybrid Action Representation.
\newblock \emph{CoRR}, abs/2109.05490.

\bibitem[{Li et~al.(2020)Li, Xing, Kirby, and Zhe}]{DBLP:conf/nips/LiXKZ20}
Li, S.; Xing, W.; Kirby, R.~M.; and Zhe, S. 2020.
\newblock Multi-Fidelity Bayesian Optimization via Deep Neural Networks.
\newblock In \emph{Advances in NeurIPS 2020, December 6-12, 2020, virtual}.

\bibitem[{Lillicrap et~al.(2016)Lillicrap, Hunt, Pritzel, Heess, Erez, Tassa, Silver, and Wierstra}]{LillicrapHPHETS15}
Lillicrap, T.~P.; Hunt, J.~J.; Pritzel, A.; Heess, N.; Erez, T.; Tassa, Y.; Silver, D.; and Wierstra, D. 2016.
\newblock Continuous control with deep reinforcement learning.
\newblock In \emph{{ICLR} 2016, San Juan, Puerto Rico, May 2-4, 2016, Conference Track Proceedings}.

\bibitem[{Ma et~al.(2020)Ma, Zhang, Xia, Zhou, Yang, and Zhao}]{dsac}
Ma, X.; Zhang, Q.; Xia, L.; Zhou, Z.; Yang, J.; and Zhao, Q. 2020.
\newblock Distributional Soft Actor Critic for Risk Sensitive Learning.
\newblock \emph{CoRR}.

\bibitem[{Martin et~al.(2017)Martin, Sasikumar, Everitt, and Hutter}]{MartinSEH17}
Martin, J.; Sasikumar, S.~N.; Everitt, T.; and Hutter, M. 2017.
\newblock Count-Based Exploration in Feature Space for Reinforcement Learning.
\newblock In \emph{Proceedings of the Twenty-Sixth IJCAI}, 2471--2478.

\bibitem[{Mavrin et~al.(2019)Mavrin, Yao, Kong, Wu, and Yu}]{DBLP:conf/icml/MavrinYKWY19}
Mavrin, B.; Yao, H.; Kong, L.; Wu, K.; and Yu, Y. 2019.
\newblock Distributional Reinforcement Learning for Efficient Exploration.
\newblock In \emph{{ICML} 2019, 9-15 June 2019, Long Beach, California, {USA}}, volume~97 of \emph{Proceedings of Machine Learning Research}, 4424--4434. {PMLR}.

\bibitem[{Nikolov et~al.(2019)Nikolov, Kirschner, Berkenkamp, and Krause}]{DBLP:conf/iclr/NikolovKBK19}
Nikolov, N.; Kirschner, J.; Berkenkamp, F.; and Krause, A. 2019.
\newblock Information-Directed Exploration for Deep Reinforcement Learning.
\newblock In \emph{{ICLR} 2019, New Orleans, LA, USA, May 6-9, 2019}. OpenReview.net.

\bibitem[{Nowozin, Cseke, and Tomioka(2016)}]{DBLP:conf/nips/NowozinCT16}
Nowozin, S.; Cseke, B.; and Tomioka, R. 2016.
\newblock f-GAN: Training Generative Neural Samplers using Variational Divergence Minimization.
\newblock In \emph{Advances in NIPS 2016, December 5-10, 2016, Barcelona, Spain}, 271--279.

\bibitem[{Osband et~al.(2016)Osband, Blundell, Pritzel, and Roy}]{DBLP:conf/nips/OsbandBPR16}
Osband, I.; Blundell, C.; Pritzel, A.; and Roy, B.~V. 2016.
\newblock Deep Exploration via Bootstrapped {DQN}.
\newblock In \emph{Advances in NIPS 2016, December 5-10, 2016, Barcelona, Spain}, 4026--4034.

\bibitem[{Pathak, Gandhi, and Gupta(2019)}]{PathakG019}
Pathak, D.; Gandhi, D.; and Gupta, A. 2019.
\newblock Self-Supervised Exploration via Disagreement.
\newblock In \emph{Proceedings of the 36th ICML, {ICML} 2019, 9-15 June 2019, Long Beach, California, {USA}}, volume~97 of \emph{Proceedings of Machine Learning Research}, 5062--5071. {PMLR}.

\bibitem[{Perrone et~al.(2019)Perrone, Shcherbatyi, Jenatton, Archambeau, and Seeger}]{DBLP:journals/corr/abs-1910-07003}
Perrone, V.; Shcherbatyi, I.; Jenatton, R.; Archambeau, C.; and Seeger, M.~W. 2019.
\newblock Constrained Bayesian Optimization with Max-Value Entropy Search.
\newblock \emph{CoRR}, abs/1910.07003.

\bibitem[{Qiu et~al.(2022)Qiu, Wang, Bai, Yang, and Wang}]{bai-UCB}
Qiu, S.; Wang, L.; Bai, C.; Yang, Z.; and Wang, Z. 2022.
\newblock Contrastive ucb: Provably efficient contrastive self-supervised learning in online reinforcement learning.
\newblock In \emph{International Conference on Machine Learning}, 18168--18210. PMLR.

\bibitem[{Rakelly et~al.(2021)Rakelly, Gupta, Florensa, and Levine}]{whichmi}
Rakelly, K.; Gupta, A.; Florensa, C.; and Levine, S. 2021.
\newblock Which Mutual-Information Representation Learning Objectives are Sufficient for Control?
\newblock \emph{CoRR}, abs/2106.07278.

\bibitem[{Savinov et~al.(2019)Savinov, Raichuk, Vincent, Marinier, Pollefeys, Lillicrap, and Gelly}]{SavinovRVMPLG19}
Savinov, N.; Raichuk, A.; Vincent, D.; Marinier, R.; Pollefeys, M.; Lillicrap, T.~P.; and Gelly, S. 2019.
\newblock Episodic Curiosity through Reachability.
\newblock In \emph{{ICLR} 2019, New Orleans, LA, USA, May 6-9, 2019}. OpenReview.net.

\bibitem[{Sutton and Barto(2018)}]{sutton2018reinforcement}
Sutton, R.~S.; and Barto, A.~G. 2018.
\newblock \emph{Reinforcement learning: An introduction}.
\newblock MIT press.

\bibitem[{Tang and Agrawal(2020)}]{Tang020Discretizing}
Tang, Y.; and Agrawal, S. 2020.
\newblock Discretizing Continuous Action Space for On-Policy Optimization.
\newblock In \emph{AAAI, New York, NY, USA, February 7-12, 2020}, 5981--5988. {AAAI} Press.

\bibitem[{Thompson(1933)}]{thompson1933likelihood}
Thompson, W.~R. 1933.
\newblock On the likelihood that one unknown probability exceeds another in view of the evidence of two samples.
\newblock \emph{Biometrika}, 25(3/4): 285--294.

\bibitem[{Wang and Jegelka(2017)}]{pmlr-v70-wang17e}
Wang, Z.; and Jegelka, S. 2017.
\newblock Max-value Entropy Search for Efficient {B}ayesian Optimization.
\newblock volume~70 of \emph{Proceedings of Machine Learning Research}, 3627--3635. International Convention Centre, Sydney, Australia: PMLR.

\bibitem[{Watkins and Dayan(1992)}]{WatkinsD92}
Watkins, C. J. C.~H.; and Dayan, P. 1992.
\newblock Technical Note Q-Learning.
\newblock \emph{Mach. Learn.}, 8: 279--292.

\bibitem[{Yuan et~al.(2023)Yuan, Hao, Ni, Mu, Zheng, Hu, Liu, Chen, and Fan}]{YuanHNMZHLCF23}
Yuan, Y.; Hao, J.; Ni, F.; Mu, Y.; Zheng, Y.; Hu, Y.; Liu, J.; Chen, Y.; and Fan, C. 2023.
\newblock {EUCLID:} Towards Efficient Unsupervised Reinforcement Learning with Multi-choice Dynamics Model.
\newblock In \emph{The Eleventh International Conference on Learning Representations, {ICLR} 2023, Kigali, Rwanda, May 1-5, 2023}. OpenReview.net.

\end{thebibliography}
\clearpage
\appendix
\section{Proof of Proposition~\ref{the} and Proposition~\ref{prop-2}}
\label{app:theorem}



In this appendix we prove Proposition~\ref{the} and Proposition~\ref{prop-2}.

\subsection{Proof of Proposition~\ref{the}}
\label{proof:themi}
In order to prove the Proposition~\ref{the}, we first propose the following lemma about ${\bf{F^{\pi}}}(s)$.
\begin{lemma}
\label{lemma1}
The integral of all mutual information between upper bound distribution of legal action $\bar{Z}^\pi(s,a')$ and policy $\pi(\cdot|s)$ at state $s$, i.e. $ {\bf{F}}^\pi(s)$ , is: 
\begin{equation}
\begin{aligned}
\label{eq4-1-2}
      {\bf{F}}^\pi(s) 
      &=\int\limits_{\substack{a}}\mathop{\mathbb{E}}_{\substack{\bar{z}(s, a)\\\sim \bar{Z}^\pi(s, a)}}  & \left[ p(a|\bar{z}(s,a), s) \log \frac{p(a|\bar{z}(s,a), s)}{\pi(a|s)}  \right] \dif{a},
\end{aligned}
\end{equation}
where $p(a|\bar{z}(s,a), s)$ represents the posterior probability distribution of policy given current state $s$ and the sampled upper bound of return $\bar{z}(s, a)$.
\end{lemma}
\begin{proof}
From the definition of mutual information, Eq.~\ref{eq:def_mi} is immediately given as:
\begin{equation}
\begin{aligned}
\label{eq-proof-theorem1-1}
&{\bf{F}}^\pi(s) \\
       = & \int\limits_{\substack{a'}}\int\limits_{\substack{a}}\int\limits_{\substack{\bar{z}(s, a')}}  \left[  p(a, \bar{z}(s,a') | s) \log \frac{p(a, \bar{z}(s,a') | s)}{p(\bar{z}(s,a'))\pi(a|s)}   \right] \\&\dif{\bar{z}(s, a')} \dif{a} \dif{a'}
    \\ =& \int\limits_{\substack{a'}}\int\limits_{\substack{a}}\int\limits_{\substack{\bar{z}(s, a')}}  \Biggl[ p(\bar{z}(s,a')) p(a|\bar{z}(s,a'), s) \\ & \log \frac{p(\bar{z}(s,a')) p(a|\bar{z}(s,a'), s)}{p(\bar{z}(s,a'))\pi(a|s)}   \Biggr] \dif{\bar{z}(s, a')} \dif{a} \dif{a'}
    \\ =& \int\limits_{\substack{a'}} \int\limits_{\substack{a}}\mathop{\mathbb{E}}_{\substack{\bar{z}(s, a')\\\sim \bar{Z}^\pi(s, a')}} \left[ p(a|\bar{z}(s,a'), s) \log \frac{p(a|\bar{z}(s,a'), s)}{\pi(a|s)}  \right] \dif{a} \dif{a'},
\end{aligned}
\end{equation}
where the posterior distribution $p(a|\bar{z}(s,a'), s)$ is the probability of choosing action $a$ on the condition of the samples from upper bounds of action $a'$.

Considering that for making decision, the probability of action $a$ is independent to the values of other actions, which means that 
\begin{equation}
    p(a|\bar{z}(s,a'), s) = 
    \begin{cases}
    \pi(a|s)&  a \neq a', \\
    p(a|\bar{z}(s,a), s) & a=a'.
    \end{cases}
\end{equation}
Therefore, Eq.~\ref{eq-proof-theorem1-1} can be further reduced as follows:
\begin{equation}
    \begin{aligned}
         {\bf{F}}^\pi(s) =\int\limits_{\substack{a}}\mathop{\mathbb{E}}_{\substack{\bar{z}(s, a)\\\sim \bar{Z}^\pi(s, a)}} &  \left[ p(a|\bar{z}(s,a), s) \log \frac{p(a|\bar{z}(s,a), s)}{\pi(a|s)}  \right] \dif{a}. \notag
    \end{aligned}
\end{equation}
\end{proof}
Lemma~\ref{lemma1} tells that ${\bf{F^{\pi}}}(s_t)$ is in direct proportion to $p(a|\bar{z}(s,a), s)$, which measures how much it is worth acting under the current policy $\pi(a|s)$ when the upper bound is known.

Next, to measure the posterior probability $p(a|\bar{z}(s,a), s)$, we use a practically effective approach of approximating the posterior probability given upper bound value \citep{pmlr-v70-wang17e, DBLP:journals/corr/abs-2009-01721, DBLP:journals/corr/abs-1910-07003, DBLP:conf/nips/LiXKZ20}.

Specifically, we approximate $p(a|\bar{z}(s,a),s)$ using the prior that $z^\pi(s,a) \leq \bar{z}(s,a)$ with given policy $\pi(s,a)$, since $\bar{z}(s,a)$ is the upper bound of $z^\pi(s,a)$. Hence, we use the indicator function $\mathds{1}_{z^\pi(s,a) \leq \bar z(s,a)}$ to truncate the policy $\pi(s,a)$, and utilize the constant C to normalize the probability, as is shown in the following equation.
\begin{equation*}
    p(a|\bar{z}(s,a),s) \approx \frac{1}{C}\pi(a|s) \mathbb{E}_{z^\pi(s, a) \sim Z^\pi(s, a)}\left[\mathds{1}_{z^\pi(s, a)\leq \bar{z}(s, a)} \right].
\end{equation*}
Here, $\mathbb{E}_{z^\pi(s, a) \sim Z^\pi(s, a)}\left[\mathds{1}_{z^\pi(s, a)\leq \bar{z}(s, a)} \right] = {\Phi_{Z^\pi(s, a)}}(\bar{z} (s, a))$, where $\Phi_{x}$ is the cumulative distribution function (CDF) of $x$, $\bar{Z}^\pi$ and $Z^{\pi}$ are the random variables, whose distributions describe the randomness of the returns, and $\bar{z}(s, a)$ is the value of random variable $\bar{Z}^\pi$. 
Therefore, the posterior probability can be measured as follows,
\begin{equation}
\label{eq4-1-6}
    p(a|\bar{z}(s,a), s) \approx  \frac{1}{C}\pi(a|s) \Phi_{Z^\pi}(\bar{z} (s, a)).
\end{equation}

In our method, we do not use approximation mechanisms about mutual information such as neural network estimation~\citep{BelghaziBROBHC18} and upper bound estimation \citep{ChengHDLGC20}. Instead, we find the correlation between random variables as shown in Eq.~\ref{eq4-1-6}, which helps to approximate mutual information directly.

According to Lemma~\ref{lemma1} and Eq.~\ref{eq4-1-6}, we can give the proof of Proposition~\ref{the} in the following.
\begin{proof}
By Combining Lemma~\ref{lemma1} and Eq.~\ref{eq4-1-6}, ${\bf{F^{\pi}}}(s)$ can be further derived as follows.
\begin{equation}
    \begin{aligned}
&{\bf{F}}^\pi(s) \\
       \approx &\int_{\substack{a\sim\pi(\cdot|s)}} \mathop{\mathbb{E}}_{\bar{z}(s, a)\sim \bar{Z}^\pi(s, a)}
       \Biggl[ \frac{1}{C}\pi(a|s)\Phi_{Z^\pi}(\bar{z} (s, a)) \\ &\log \frac{\pi(a|s)\Phi_{Z^\pi}(\bar{z} (s, a))}{C\pi(a|s)}  \Biggr] \dif{a} \notag\\
       =&\frac{1}{C} \mathop{\mathbb{E}}_{\substack{a\sim\pi(\cdot|s) \\ \bar{z}(s, a)\sim \bar{Z}^\pi(s, a)}}\left[\Phi_{Z^\pi}(\bar{z} (s, a)) \log \frac{\Phi_{Z^\pi}(\bar{z} (s, a))}{C}  \right]  \notag
    \end{aligned}
\end{equation}
Here, the last equality follows from Proposition~\ref{the}.
\end{proof}

\subsection{Proof of Proposition~\ref{prop-2}}
\label{proof:prop}
\begin{proof}
Similar to~\cite{DBLP:conf/nips/CiosekVLH19}, we set the covariance matrix of $\pi_E$ is that of $\pi_\phi$, i.e., $\Sigma_E = \sigma_\phi$. Hence, the OVD-Explorer problem is simplified as finding the $\mu$ that maximizes $\mathbf{\hat{F}}^\pi(s, \mu)$:
\begin{align}
\mathbf{\hat{F}}^\pi(s, \mu) = \mathop{\mathbb{E}}_{ \bar{Z}^\pi}\left[\Phi_{Z^\pi}(\bar{z} (s, \mu)) \log \frac{\Phi_{Z^\pi}(\bar{z} (s, \mu))}{C}  \right] \notag
\end{align}

To ensure that $\pi_E$ samples actions around $\pi_\phi$, we derive $\pi_E$ upon mean $\mu_\phi$ of target policy $\pi_\phi$. In specific, we firstly obtain the gradient of $\mathbf{\hat{F}}(s,\mu)$ at $\pi_\phi$, which is given as follows:
\begin{equation}
    \nabla_a \mathbf{\hat{F}}^\pi(s, \mu)|_{\mu=\mu_{\phi}} = \mathbb{E}_{\bar{Z}^\pi} \left[\hat{m} \times \frac{\partial \bar{z} (s, a)}{\partial a}|_{a=\mu_{\phi}} \right] \notag
\end{equation}
where $\hat{m}= \phi_{Z^\pi(s, \mu_{\phi})}( \bar{z} (s, \mu_{\phi}))(\log \frac{\Phi_{Z^\pi(s, \mu_{\phi})}(\bar{z} (s, \mu_{\phi}))}{C}+1)$,
and $\phi(x)$ is the probability distribution function (pdf). Hence, $\mu_E$ is given as follows:
\begin{equation}
 \mu_E=\mu_{\phi}+\alpha\mathbb{E}_{\bar{Z}^\pi}\left[m \times \frac{\partial \bar{z} (s, a)}{\partial a}|_{a=\mu_{\phi}}\right], \notag
\end{equation}
where $\alpha$ is the step size controlling exploration level and $m =  \log \frac{\Phi_{Z^\pi(s, \mu_{\phi})}(\bar{z} (s, \mu_{\phi}))}{C}+1$.
\end{proof}

\section{Details about OVD-Explorer}
\subsection{The formulation of uncertainties}
Our proposed method considers two types of uncertainty that exist within the RL system, both of which play a crucial role in our approach. Below, we elaborate on the formulation of these uncertainties.

The epistemic uncertainty characterizes the ambiguity of the model arisen from insufficient knowledge, and it tends to be high at state-action pairs that are rarely visited. To estimate the epistemic uncertainty, we leverage the disagreement among ensemble estimators~\cite{DBLP:conf/nips/OsbandBPR16}:
\begin{equation}
\begin{aligned}
\sigma_{\text{epistemic}}^2(s, a) = \mathbb{E}_{i \sim \mathcal{U}(1, N)} \text{var}_{k=1, 2}\hat{Z}_{\tau_i}(s, a; \theta_k),
\end{aligned}
\end{equation}
where $\mathcal{U}$ is uniform distribution, $N$ is the number of quantiles, and $\hat{Z}_{\tau_i}(s,a;\theta_k)$ is the value of the $i$-th quantile drawn from $\hat{Z}(s,a;\theta_k)$.

The aleatoric uncertainty~(noise) arises from the randomness in the environment, which can be attributed to the stochastic nature of policies, rewards, and/or transition probabilities. To model the aleatoric uncertainty (noise), we consider the variance of the value distribution $\bar{Z}^\pi$~\citep{abs-1905-09638} as follows:
\begin{equation}
\label{eq:alea}
     \sigma_{\text{aleatoric}}^2(s, a) = \text{var}_{i \sim \mathcal{U}(1, N)}\left[\mathbb{E}_{k=1, 2}\hat{Z}_{\tau_i}(s, a;\theta_k)\right].
\end{equation}

\subsection{Algorithm 2: OVD-Explorer for SAC}
\label{app:dsac}

the behavior policy generated by OVD-Explorer can be seamlessly integrated with existing policy-based RL algorithms from a distributional perspective, thereby promoting stable and effective exploration. For this purpose, in the context of SAC, we need to estimate the value distribution, resulting in the distributional variant of SAC, referred to as DSAC~\citep{dsac}. Subsequently, we replace the behavior policy with the one generated by Algorithm~\ref{alg1} to interact with the environment. The complete algorithm for our implementation of OVD-Explorer based on SAC is presented in Algorithm~\ref{algo2}. The entire code can be accessed in the supplementary material.
\label{alg2}
\begin{algorithm}[h]
    \caption{OVD-Explorer for DSAC}
    \label{algo2}
\begin{algorithmic}[1]
    \STATE {\bfseries Initialise:} Value networks $\theta_1$, $\theta_2$, policy network $\phi$ and their target networks $\bar{\theta}_1$, $\bar{\theta}_2$, $\bar{\phi}$, quantiles number N, target smoothing coefficient ($\tau$), discount ($\gamma$), an empty replay pool $\mathcal{D}$
    \FOR{each iteration}
        \FOR{each environmental step}
            \STATE $a_t \sim \pi_E(a_t, s_t)$ according to Algorithm~1
            \STATE $\mathcal{D} \leftarrow \mathcal{D} \cup \{(s_t, a_t, r(s_t, a_t), s_{t+1})\}$
        \ENDFOR
        \FOR{each training step}
            \FOR{i = 1 to N}
                \FOR{j = 1 to N}
                    \STATE calculate $ \delta_{i, j}^k, k = {1, 2}$, following Eq.~\ref{3-3-e5}
                \ENDFOR
            \ENDFOR
            \STATE Calculate $\mathcal{L}_{QR}(\theta_k), k = {1, 2}$ using $\delta_{i, j}^k$ following Eq.~\ref{3-2-e2}
            \STATE Update $\theta_k$ with $\nabla \mathcal{L}_{QR}(\theta_k)$
            \STATE Calculate $\mathcal{J}_\pi(\phi)$, following Eq.~\ref{3-3-e7}
            \STATE Update $\phi$ with $\nabla\mathcal{J}_\pi(\phi)$
        \ENDFOR
        \STATE Update target value network with $\bar{\theta}_k \leftarrow \tau \theta_k + (1-\tau)\bar{\theta}_k, k=1, 2$
        \STATE Update target policy network with $\bar{\phi} \leftarrow \tau \phi + (1-\tau)\bar{\phi}$
    \ENDFOR
    \end{algorithmic}
\end{algorithm}

\section{Detailed Experimental Settings}
\subsection{Baseline Algorithms and Implementation Details}

As mentioned in Sec~\ref{4-3}, OVD-Explorer can be integrated with any policy-based DRL algorithm, such as SAC \citep{HaarnojaZAL18SAC} or TD3 \citep{FujimotoHM18}, by constructing optimistic and current value distributions reasonably. In this paper, we implement OVD-Explorer based on SAC. In the following, we first elaborate on the baseline algorithms and implementation details for better reproducibility. 

The baseline algorithms include SAC~\citep{HaarnojaZAL18SAC}, DSAC~\citep{dsac}, and DOAC.
\cite{dsac} has compared the performance of the distributional extension of SAC and TD4 (i.e., DSAC and TD4, respectively), showing that DSAC outperforms TD4 on Mujoco tasks. Thus, we implement OVD-Explorer based on SAC to show its advantage of exploration, comparing with SAC and DSAC. 
Besides, DOAC, the distributional variant of OAC \citep{DBLP:conf/nips/CiosekVLH19}, performing optimistic exploration but ignoring the noise, is involved to illustrate the necessity of noise-aware exploration.

\textit{SAC.} We implement SAC \citep{HaarnojaZAL18SAC} based on the OAC repository\footnote{https://github.com/microsoft/oac-explore}. The results in Ant-v2 and Hopper-v2 are similar to reported results, and we report a better result than OAC's implementation for SAC on HalfCheetah-v2. 

\textit{DSAC.} DSAC is implemented based on SAC, except that the distributional $Q$ function is used instead of the traditional $Q$ function in SAC. We set same hyper-parameters for DSAC and SAC to ensure a fair comparison. In our results, DSAC can guarantee an absolute advantage over SAC in most cases, which is consistent with the previous conclusion. \par

\textit{DOAC.} We implement DOAC based on DSAC as well as the OAC repository. As DSAC shows great advantage due to the distributional value estimation, to ensure a fair comparison, we extend OAC \citep{DBLP:conf/nips/CiosekVLH19} to its distributional version, i.e., DOAC, by replacing the exploration process of DSAC by the behavior policy derived by OAC. We set the hyper-parameters the same as used by OAC in Mujoco,\footnote{That is given by the open source code, where $\beta_{\text{UB}}$ is 4.66 and $\delta$ is 23.53.} and our results of DOAC on Ant-v2 and HalfCheetah-v2 are significantly better than that OAC reported.\par
Finally, we illustrate the summary of the baselines with respect to the optimistic and noise-aware exploration in Table~\ref{tab:baselines}. OVD-Explorer possesses both of these essential exploration attributes, making it a promising and effective method in noisy environments. DOAC demonstrates optimistic exploration but lacks noise-aware capabilities. DSAC and SAC, on the other hand, do not exhibit either optimistic or noise-aware exploration characteristics.
\begin{table}[h]
\caption{Summary of OVD-Explorer and baselines regarding to the optimistic and risk-averse exploration. }
\small
\begin{center}
\begin{tabular}{ccc}
\toprule
\multirow{1}{*}{Algorithm} & Optimistic & Risk-Averse \tabularnewline
\midrule
OVD-Explorer& $\checkmark$ & $\checkmark$ \tabularnewline
DOAC & $\checkmark$ & $\times$ \tabularnewline
DSAC & $\times$ & $\times$ \tabularnewline
SAC & $\times$ & $\times$ \tabularnewline
\bottomrule
\end{tabular}
\end{center}
\label{tab:baselines}
\end{table}

\begin{table*}[t]
\small
\begin{center}
\begin{tabular}{p{2.0cm}p{4.9cm}c|p{4.6cm}}
\toprule
\multicolumn{1}{c}{}  &\multicolumn{1}{c}{\bf Parameter}  &\multicolumn{1}{c|}{} &\multicolumn{1}{c}{\bf Value}
\\\midrule
Training & Discount  & &0.99 \\
 & Target smoothing coefficient  &$\tau$ &5e-3\\
  & Learning rate  & &3e-4\\
  & Optimizer  & &Adam \citep{KingmaB14}\\
   & Batch size  & &256\\
   & Quantiles amount  & &20\\
  & Replay buffer size  & &$1.0\times 10^6$ for Mujoco tasks \\
  &   & &$1.0\times 10^5$ for other tasks \\
  & Environment steps per epoch & & $1.0\times 10^3$ for Mujoco tasks\\
  &  & & $1.0\times 10^2$ for other tasks\\
\midrule
Exploration & Exploration ratio  &$\alpha$ &0.05 (specified otherwise)\\
 & Uncertainty ratio  &$\beta$ &3.2 \\
 & Normalization factor  &$C$ &0.5\\
 \bottomrule
\end{tabular}
\end{center}
\caption{Hyper-parameters in OVD-Explorer}
\label{tab:hyper-parameters}
\end{table*}
\subsection{Hyper-parameters}
We illustrate the gyper-parameters used for training and exploration in Table~\ref{tab:hyper-parameters}.

\section{Discussions about other related works}

\textit{Mutual information used in exploration.} Following OFU principle, OVD-Explorer uses mutual information to define exploration ability for guiding exploration. 
There are some other information-theoretic exploration strategies, 
such as VIME \citep{houthooft2016vime}, which measures the information gain on environment dynamics, and EMI \citep{KimKJLS19EMI}, generating intrinsic reward using prediction error of representation learned by mutual information. 
which can solve sparse reward problem well using intrinsic reward. 
Nevertheless, those methods use mutual information neither on the value distribution, nor for OFU-based exploration. Besides, instead of using the mechanisms for approximating mutual information~\citep{whichmi}, such as variational inference \citep{DBLP:conf/colt/HintonC93,houthooft2016vime} or f-divergence \citep{DBLP:conf/nips/NowozinCT16,KimKJLS19EMI}, we find the correlation between policy and upper bounds distribution through uncertainty, which helps estimate mutual information directly.

\textit{Safe exploration.} Another related line of work focuses on risk-averse exploration or safe exploration, from the perspective of safe RL. Those methods design exploration strategies to achieve risk-averse policy (i.e., CVaR policy)~\citep{efficient_risk_averse, KeramatiDTB20risk-averse}, or the policy with safety constraint~\citep{safeexplorationc20118, DingWYWJ21safeexp}. Those methods have a fundamental difference with our approach. Specifically, our method aims to train the policy that maximizes expected return, concerned with the optimistic exploration in complex stochastic continuous control tasks, rather than a safe policy.

\section{More Results}

In this section, we present a more comprehensive set of experimental results, covering various aspects that are not included in the main text due to space constraints. The additional results are as follows:
\begin{enumerate}
    \item Time Consumption: We provide an analysis of the time consumption of OVD-Explorer in comparison to the baseline algorithms.
    \item Hyperparameter Sensitivity: We investigate the impact of the two hyperparameters in OVD-Explorer and their influence on the algorithm's performance.
    \item Comparison with Non-Uncertainty-Based Strategies: We compare OVD-Explorer with other non-uncertainty-based exploration strategies to demonstrate its superiority in handling noisy environments effectively.
    \item Results on Harder Noisy Tasks: We showcase the performance of OVD-Explorer on more challenging noisy tasks to assess its robustness in highly complex environments.
    \item Exploration Patterns Visualization: We visualize the exploration patterns of OVD-Explorer by tracking state space visitation frequencies and similar measures. This provides insights into how the algorithm adapts its exploration strategy during the learning process.
    \item Comparison of Exploration Patterns: We analyze the differences in exploration patterns between scenarios with high and low noise levels near the goal in the GridChaos task. This analysis sheds light on how OVD-Explorer explores noisy regions differently depending on the noise intensity.
\end{enumerate}

\subsection{Runtime Analysis}

\begin{table}[hbt]
    \centering
    \small
    \begin{tabular}{>{\raggedright}m{2.0cm}>{\centering}m{2.5cm}}
        \toprule
        Method & Relative Runtime Time \tabularnewline 
        \midrule
        SAC & 1.0  \tabularnewline
        DSAC & 1.17  \tabularnewline
        DOAC & 1.19  \tabularnewline
        OVDE\_G & 1.17  \tabularnewline
        OVDE\_Q & 1.21  \tabularnewline
        \bottomrule
    \end{tabular}
    \caption{Computational costs.}
    \label{tab:cost_cmp}
\end{table}
Table~\ref{tab:cost_cmp} shows the time consumption of algorithms relative to SAC.
As can be seen, the distributional value estimation used in DSAC, DOAC and our methods introduces extra time consumption distinctly. Nevertheless, the relative time consumption of OVDE\_G and OVDE\_Q spends approximately 17\% to 21\% more time than SAC to achieve a significant performance gain of nearly 100\%, as demonstrated in Figure 6(b). This indicates that the extra time consumption of OVD-Explorer is well justified by the substantial improvement in exploration efficiency.
Besides, the time consumption of OVDE\_Q is close to that of DSAC, with only a slightly larger variance. This suggests that the additional time consumption of OVDE\_Q is minimal while still achieving better exploration performance, making it a promising choice for practical applications.

\begin{table*}[h]
\scriptsize
\begin{center}
\begin{tabular}{llllll}
\toprule
Task   & DSAC  & DSAC+RND & IDS & OVDE\_G & OVDE\_Q \\
\midrule
Ant-v2                     &  6385.9$\pm$1287.2   & 7308.4$\pm$641.3 & 503.2$\pm$23.3 & 7175.3$\pm$789.0          & \textbf{7382.3}$\pm$466.6  \\
HalfCheetah-v2             & 13348.4$\pm$1957.1  & 12198.1$\pm$2338.3  & 203.3$\pm$33.2 & 14796.2$\pm$1473.2  & \textbf{16484.3}$\pm$1373.75   \\
Hopper-v2                  & 2506.0$\pm$390.56    & 2077.9$\pm$344.1  & 1201.0$\pm$42.2 & 2394.6$\pm$496.6          & \textbf{2559.3}$\pm$384.5  \\
N-HalfCheetah-v2       & 431.81$\pm$39.41    & 409.48$\pm$45.88  & 68.3$\pm$21.3 & 447.3$\pm$38.57   & \textbf{453.56}$\pm$55.97  \\
N-Hopper-v2            & 236.62$\pm$19.89     & 231.46$\pm$9.94  & 136.3$\pm$58.1 & {234.88}$\pm$15.24   & \textbf{239.43}$\pm$9.90  \\
N-Ant-v2 (250)          & 1217.87$\pm$185.00  & 1306.05$\pm$223.18 & 36.3$\pm$63.1 & \textbf{1434.61}$\pm$113.53 & {1340.00}$\pm$221.64   \\

\bottomrule
\end{tabular}
\caption{Comparisons with IDS and RND. For DSAC+RND and IDS, the averaged performance and standard deviation of 5 runs are reported.}
\label{tab:performance-mujoco-RND}
\end{center}
\end{table*}

\subsection{Sensitivity to $\alpha$ and $\beta$}
\label{sec:rq3}

\begin{figure}[h]
    \centering
    \quad \hspace{-5mm}
    \subfigure[$\alpha$ value]{
        \includegraphics[width=3.6cm]{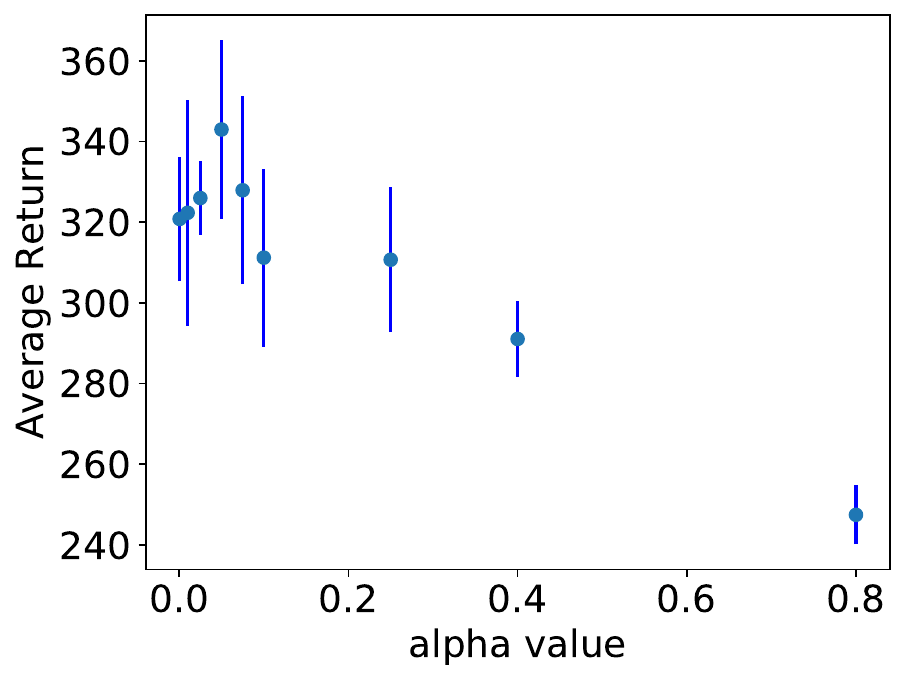}
    }
    \quad \hspace{-5mm}
    \subfigure[$\beta$ value]{
        \includegraphics[width=3.6cm]{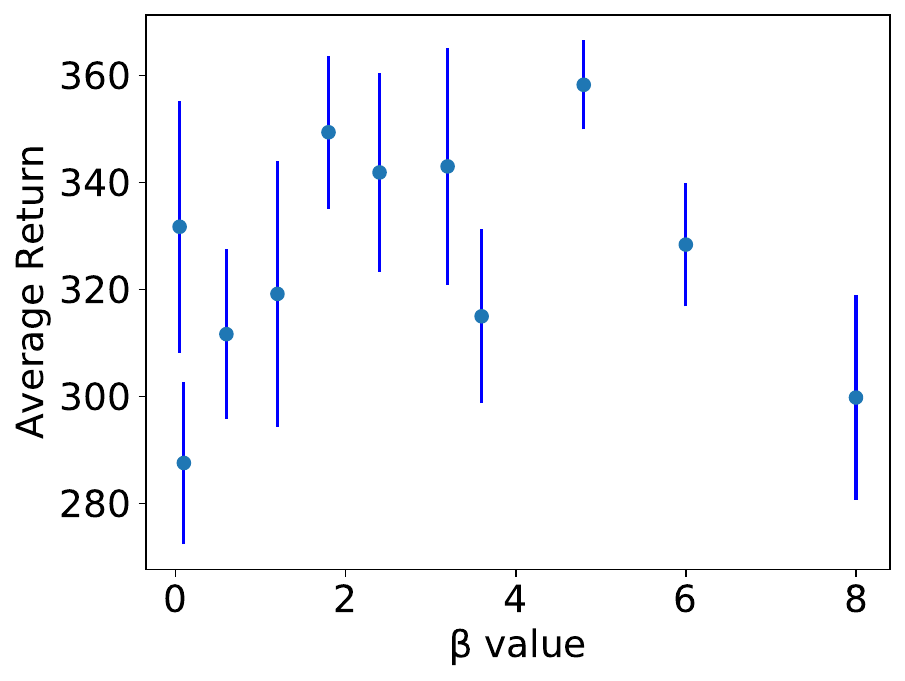}
    }
    \caption{Sensitivity to (a) $\alpha$ and (b) $\beta$. The x-axis indicates values of hyper-parameters, while the y-axis is the evaluation result represented by average episode return and half standard deviation (error bars) over 5 seeds. The 9 different $\alpha$ values are 0.0005, 0.01. 0.025, 0.05, 0.075, 0.1, 0.25, 0.4, 0.8, respectively. The 11 different $\beta$ values are 0.05, 0.1, 0.6, 1.2, 1.8, 2.4, 3.2, 3.6, 4.8, 6.0, 8.0.}
    \label{fig:a_b_exp}
\end{figure}
To determine the appropriate values for the hyperparameters $\alpha$ and $\beta$ in OVD-Explorer, we conduct experiments to investigate their impact on the algorithm's performance.

For the hyperparameter $\alpha$, which controls the distance between the exploration policy $\pi_E$ derived from OVD-Explorer and the given policy $\pi_\phi$, we test several $\alpha$ values on the Noisy Ant-v2 task using OVDE\_G. The results are shown in Fig.~\ref{fig:a_b_exp}(a). It is observed that the performance of OVD-Explorer varies with different $\alpha$ values. If $\alpha$ is too small, OVD-Explorer degenerates to DSAC and lacks optimistic exploration. Conversely, if $\alpha$ is too large, the performance worsens due to a substantial gap between $\pi_E$ and $\pi_\phi$. However, there exists a concentrated range of $\alpha$ values that facilitate more efficient exploration. In our experiments, we uniformly use $\alpha = 0.05$ to showcase the performance of OVD-Explorer.

Regarding the hyperparameter $\beta$, which controls the scale of uncertainty quantification in Eq.~\ref{4-2-e13} and Eq.~\ref{4-2-e15}, influencing $\bar{Z}^\pi$ and $Z^{\pi}$, we conduct experiments on the Noisy Ant-v2 task using OVDE\_G. The sensitivity analysis of $\beta$ is illustrated in Fig.~\ref{fig:a_b_exp}(b). The results demonstrate a broad range of suitable values for $\beta$ that lead to good performance. For our experiments, we set $\beta$ uniformly to be 3.2.

\subsection{Comparison with Other Similar Algorithms}
\label{sec:rq4}
OVD-Explorer achieves optimistic noise-aware exploration, and outperforms the baseline algorithms including SAC, DSAC and DOAC. Additionally, there are other algorithms that share similar or partial objectives with OVD-Explorer, such as optimistic exploration based on intrinsic motivation and optimistic noise-aware exploration for discrete control. To provide a comprehensive comparison, we evaluate OVD-Explorer against these algorithms on three standard Mujoco tasks and three noisy Mujoco tasks, as shown in Table~\ref{tab:performance-mujoco-RND}.

\begin{figure*}[t]
\centering
\subfigure[100 steps]{
\includegraphics[width=4.8cm]{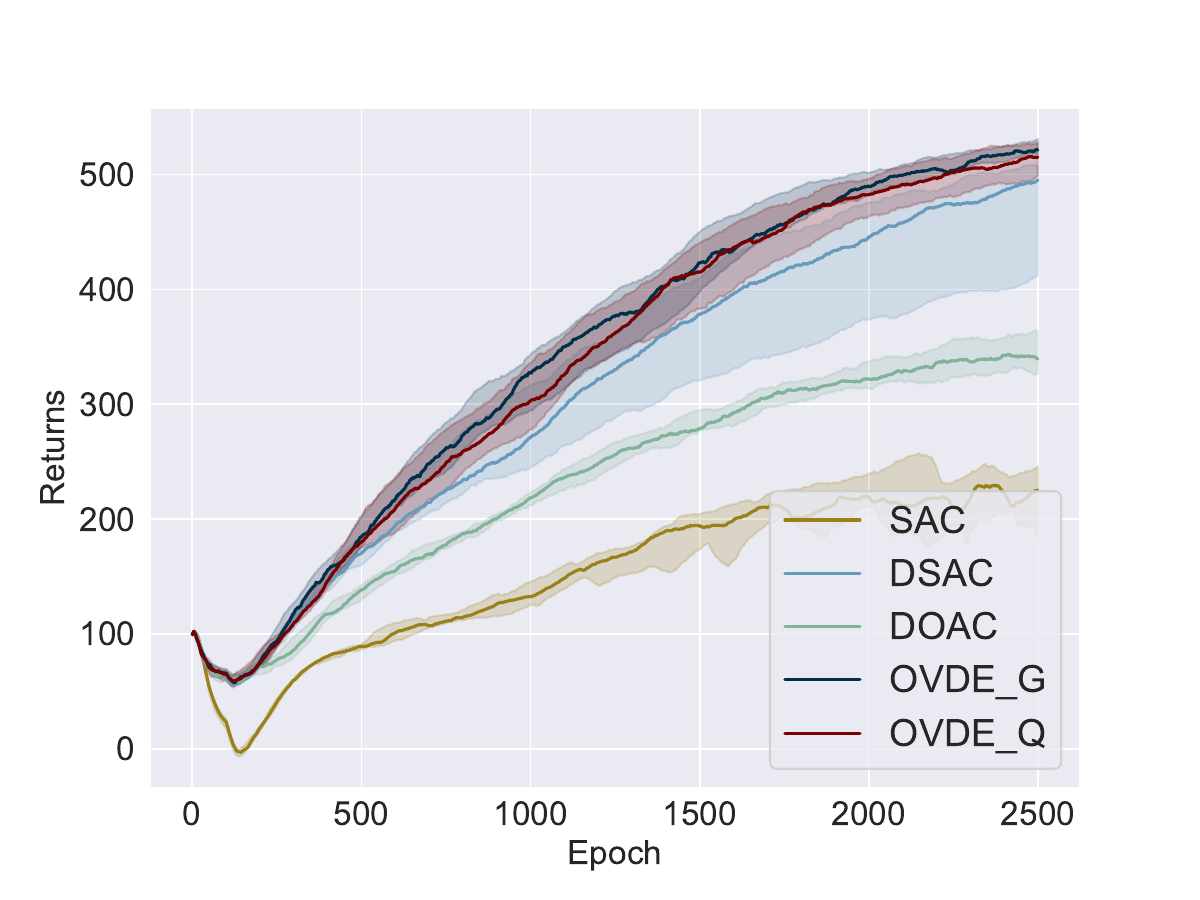}
}
\quad \hspace{-12mm}
\subfigure[250 steps]{
\includegraphics[width=4.8cm]{figs/nant250m.pdf}
}
\quad \hspace{-12mm}
\subfigure[500 steps]{
\includegraphics[width=4.8cm]{figs/nant500_0.05m.pdf}
}
\\
\subfigure[500 steps]{
\includegraphics[width=4.8cm]{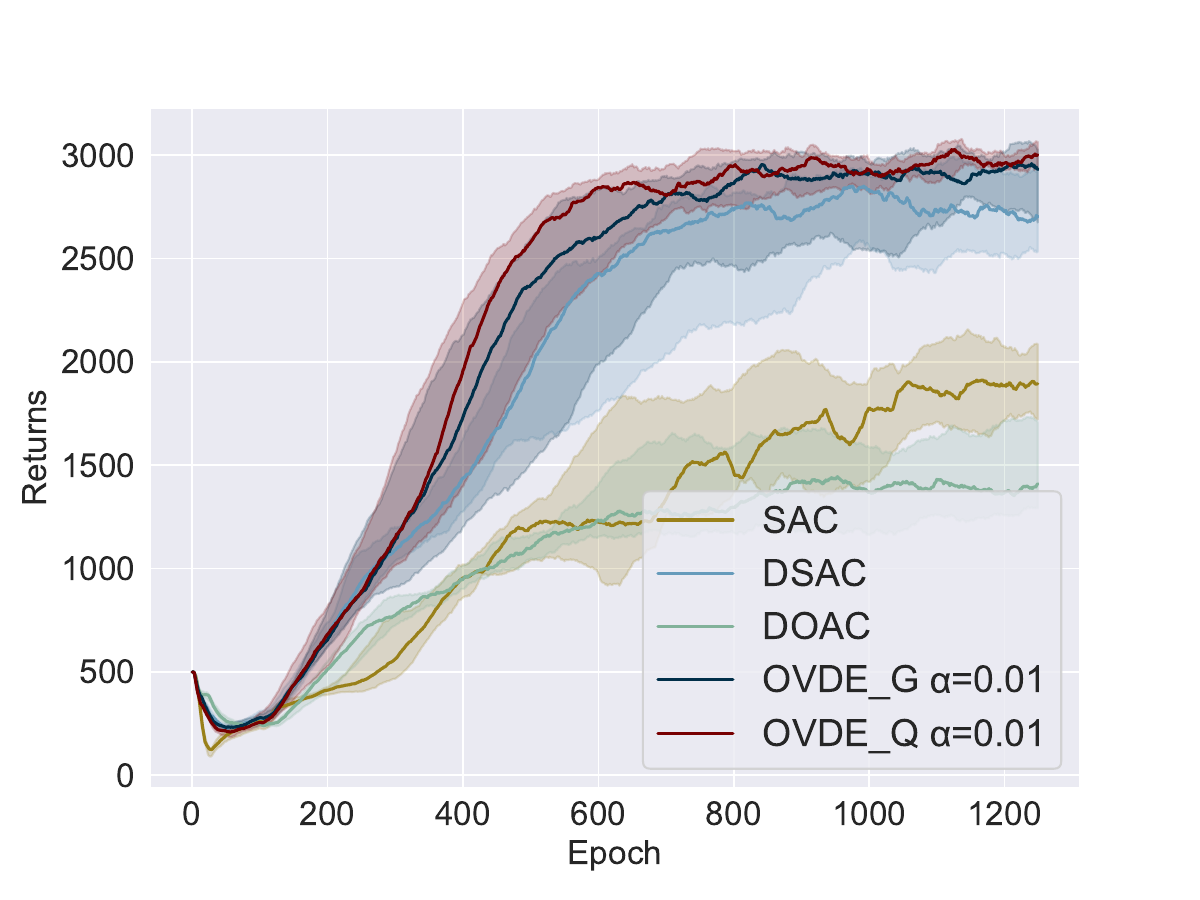}
}
\quad \hspace{-12mm}
\subfigure[750 steps]{
\includegraphics[width=4.8cm]{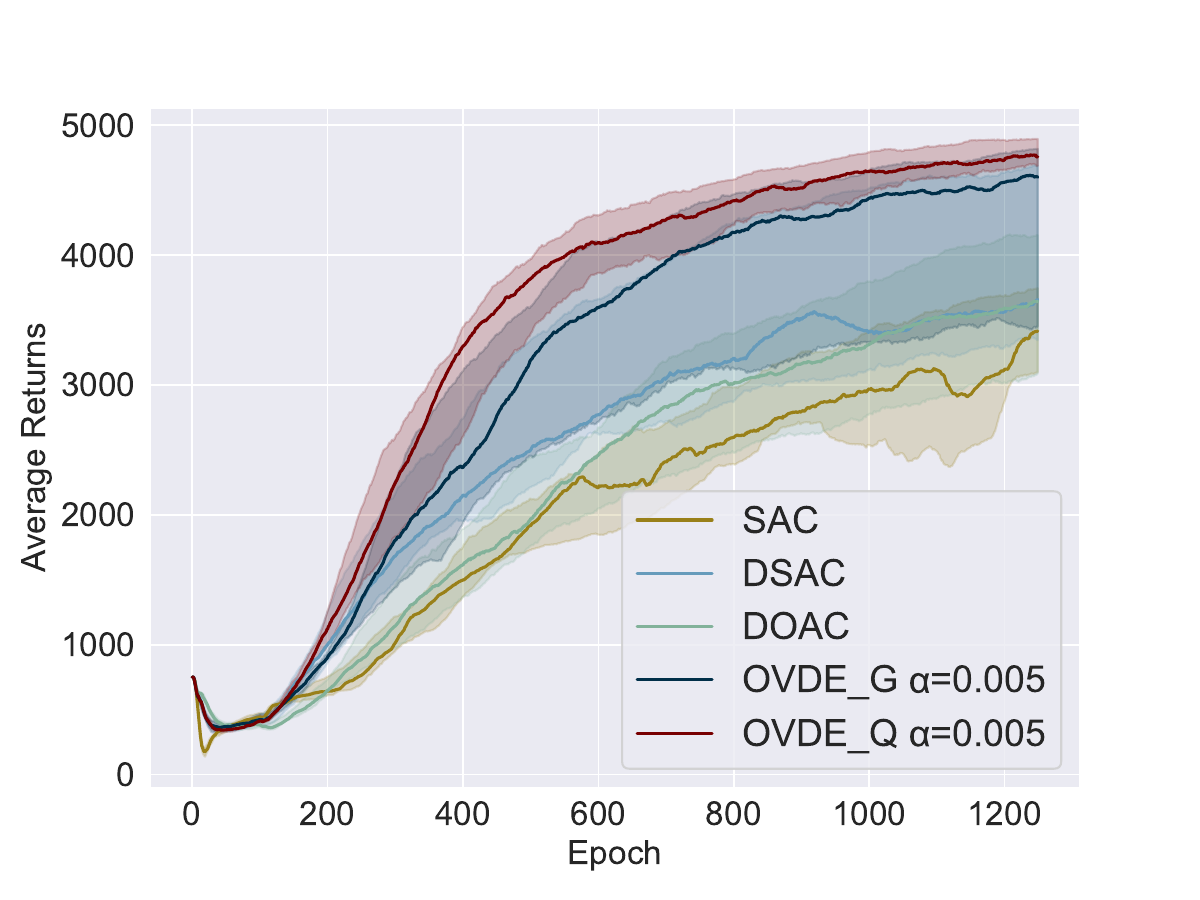}
}
\quad \hspace{-12mm}
\subfigure[1000 steps]{
\includegraphics[width=4.8cm]{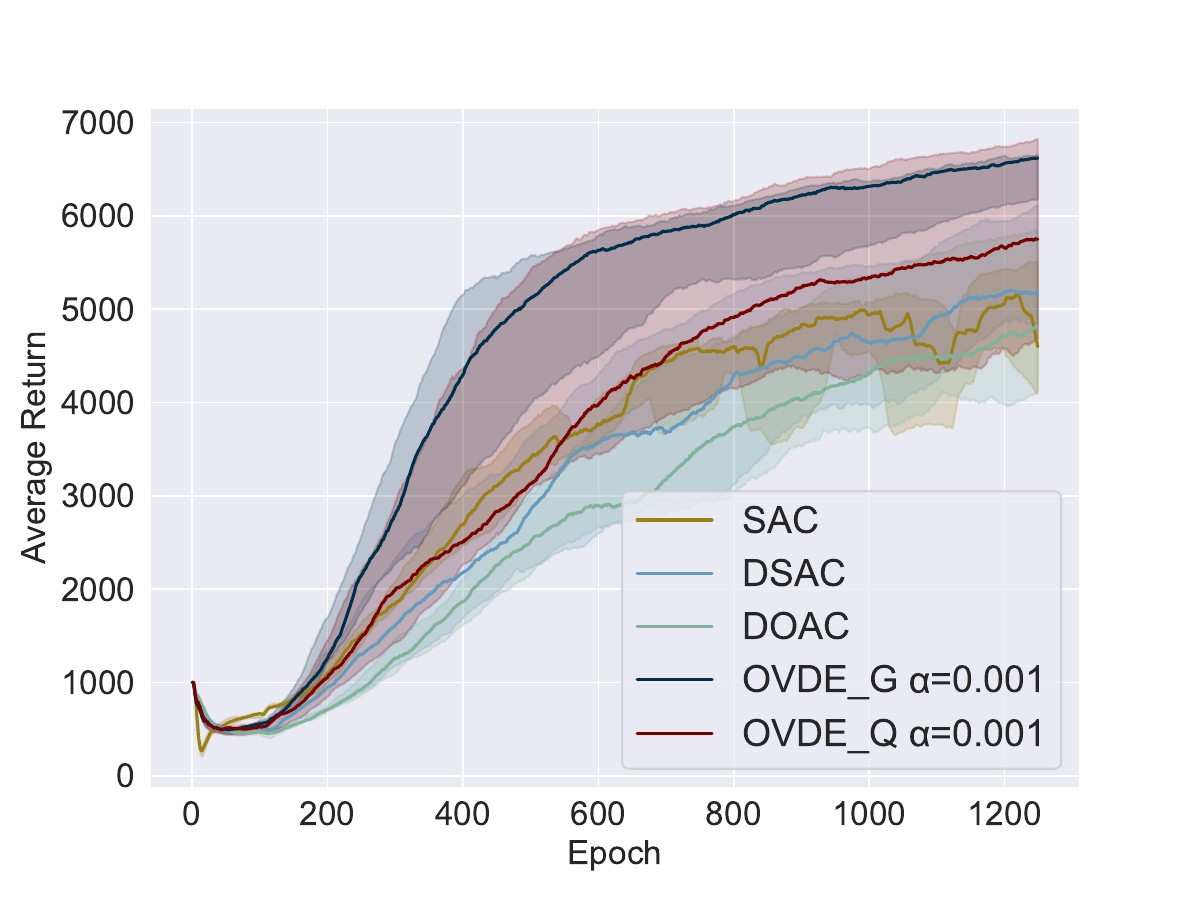}
}

\caption{Training curves on Noisy Ant-v2 tasks with different maximum episodic length setup. The sub-title of each figure represents the episodic horizon, also known as the maximum episode length. We report the median of returns and \textbf{the interquartile range} of 10 runs. Curves are smoothed uniformly for visual clarity.}
\label{fig:ablation-horizon}
\end{figure*}
\subsubsection{RND}
Indeed, RND also follows the principle of Optimism in the Face of Uncertainty (OFU) and utilizes uncertainty estimation through network distillation to drive intrinsic motivation for agent exploration. To ensure fairness in comparison, we implement RND based on DSAC, denoted as DSAC+RND, and present the results in the DSAC+RND column of Table~\ref{tab:performance-mujoco-RND}.

The results clearly indicate that OVD-Explorer maintains a significant performance advantage over DSAC+RND. While DSAC+RND shows promising results on the Ant-v2 task, it does not perform as well on other tasks. This performance gap can be attributed to two main factors.

Firstly, intrinsic motivation-based exploration methods, like RND, are known to exhibit some instability and sensitivity to the scale of intrinsic reward signals. This instability can result in varying performance when applied to different tasks, leading to inconsistent outcomes. Secondly, RND focuses primarily on episodic uncertainty, which may not be as effective in tasks that involve random noise, where the algorithm lacks the ability to adequately manage aleatoric uncertainty. In contrast, OVD-Explorer's ability to handle both episodic and aleatoric uncertainty gives it an edge in effectively addressing noisy exploration challenges.

\subsubsection{Noise-aware Optimistic Exploration for Discrete Control}

For discrete control problems, IDS \citep{DBLP:conf/iclr/NikolovKBK19} exhibits similar optimistic exploration abilities as OVD-Explorer and can also avoid the negative impact of heterogeneous noise. However, it is important to note that IDS is specifically designed for discrete control tasks and may face limitations when applied to continuous control problems.

To assess the performance of IDS in solving continuous control problems, we discretize the continuous action space and present the evaluation results in the IDS column of Table~\ref{tab:performance-mujoco-RND}. The results demonstrate that especially on tasks with larger action dimensions, IDS performs extremely poorly due to dimension explosion.

This observation highlights the significance of designing exploration strategies tailored for continuous reinforcement learning, such as OVD-Explorer. OVD-Explorer's ability to guide optimistic exploration while avoiding excessive exploration in noisy areas proves essential for achieving robust and efficient exploration in continuous control tasks. As evidenced by the results, algorithms like OVD-Explorer play a crucial role in addressing the unique challenges posed by continuous control problems and outperforming approaches designed solely for discrete control scenarios.

\subsection{Evaluation on Harder Noisy Tasks in Mujoco}
\begin{table*}[t]
\centering
\begin{tabular}{>{\raggedright}m{0.8cm}>{\centering}p{0.78cm}>{\centering}p{2.0cm}>{\centering}p{2.0cm}>{\centering}p{2.0cm}>{\centering}p{2.2cm}>{\centering}p{2.2cm}}
\toprule
\centering{}{\scriptsize{}Horizon} & \centering{}{\scriptsize{}$\alpha$} & \centering{}{\scriptsize{}SAC} & \centering{}{\scriptsize{}DSAC} & \centering{}{\scriptsize{}DOAC} & \centering{}{\scriptsize{}OVDE\_G} & \centering{}{\scriptsize{}OVDE\_Q}\tabularnewline
\midrule
\multirow{2}{1.1cm}{\centering{}{\scriptsize{}100}} & \multirow{2}{0.78cm}{\centering{}{\scriptsize{}0.05}} & \centering{}{\scriptsize{}222.96\textpm 41.93} & \centering{}{\scriptsize{}465.34\textpm 53.94} & \centering{}{\scriptsize{}344.71\textpm 20.39} & \centering{}{\scriptsize{}\textbf{524.16\textpm 10.54}} & \centering{}{\scriptsize{}513.77\textpm 17.87}\tabularnewline
\cline{3-7} \cline{4-7} \cline{5-7} \cline{6-7} \cline{7-7} 
 &  & \centering{}{\scriptsize{}\textless0.001} & \centering{}{\scriptsize{}0.013} & \centering{}{\scriptsize{}\textless0.001} & \centering{}{\scriptsize{}-} & \centering{}{\scriptsize{}0.406}\tabularnewline
\hline 
\multirow{2}{1.1cm}{\centering{}{\scriptsize{}250}} & \multirow{2}{0.78cm}{\centering{}{\scriptsize{}0.05}} & \centering{}{\scriptsize{}764.93\textpm 159.5} & \centering{}{\scriptsize{}1217.87\textpm 185.00} & \centering{}{\scriptsize{}840.26\textpm 115.91} & \centering{}\textbf{\scriptsize{}1434.61\textpm 113.53} & \centering{}{\scriptsize{}1340.00\textpm 221.64}\tabularnewline
\cline{3-7} \cline{4-7} \cline{5-7} \cline{6-7} \cline{7-7} 
 &  & \centering{}{\scriptsize{}\textless0.001} & \centering{}{\scriptsize{}0.002} & \centering{}{\scriptsize{}\textless0.001} & \centering{}{\scriptsize{}-} & \centering{}{\scriptsize{}0.414}\tabularnewline
\hline 
\multirow{2}{1.1cm}{\centering{}{\scriptsize{}500}} & \multirow{2}{0.78cm}{\centering{}{\scriptsize{}0.05}} & \centering{}{\scriptsize{}1788.01\textpm 480.22} & \centering{}{\scriptsize{}2613.02\textpm 400.11} & \centering{}{\scriptsize{}1523.75\textpm 368.39} & \centering{}{\scriptsize{}2864.40\textpm 351.54} & \centering{}\textbf{\scriptsize{}2951.81\textpm 230.41}\tabularnewline
\cline{3-7} \cline{4-7} \cline{5-7} \cline{6-7} \cline{7-7} 
 &  & \centering{}{\scriptsize{}\textless0.001} & \centering{}{\scriptsize{}0.035} & \centering{}{\scriptsize{}\textless0.001} & \centering{}{\scriptsize{}0.194} & \centering{}{\scriptsize{}-}\tabularnewline
\hline 
\multirow{2}{1.1cm}{\centering{}{\scriptsize{}500}} & \multirow{2}{0.78cm}{\centering{}{\scriptsize{}0.01}} & \centering{}{\scriptsize{}1788.01\textpm 480.22} & \centering{}{\scriptsize{}2613.02\textpm 400.11} & \centering{}{\scriptsize{}1523.75\textpm 368.39} & \centering{}{\scriptsize{}2831.99\textpm 283.31} & \centering{}\textbf{\scriptsize{}3006.58\textpm 78.92}\tabularnewline
\cline{3-7} \cline{4-7} \cline{5-7} \cline{6-7} \cline{7-7} 
 &  & \centering{}{\scriptsize{}\textless0.001} & \centering{}{\scriptsize{}0.030} & \centering{}{\scriptsize{}\textless0.001} & \centering{}{\scriptsize{}0.619} & \centering{}{\scriptsize{}-}\tabularnewline
\hline 
\multirow{2}{1.1cm}{\centering{}{\scriptsize{}750}} & \multirow{2}{0.78cm}{\centering{}{\scriptsize{}0.005}} & \centering{}{\scriptsize{}3208.58\textpm 676.28} & \centering{}{\scriptsize{}3932.89\textpm 738.76} & \centering{}{\scriptsize{}3601.00\textpm 557.70} & \centering{}{\scriptsize{}4207.39\textpm 721.80} & \centering{}\textbf{\scriptsize{}4469.29\textpm 744.26}\tabularnewline
\cline{3-7} \cline{4-7} \cline{5-7} \cline{6-7} \cline{7-7} 
 &  & \centering{}{\scriptsize{}\textless0.001} & \centering{}{\scriptsize{}0.146} & \centering{}{\scriptsize{}0.003} & \centering{}{\scriptsize{}0.330} & \centering{}{\scriptsize{}-}\tabularnewline
\hline 
\multirow{2}{1.1cm}{\centering{}{\scriptsize{}1000}} & \multirow{2}{0.78cm}{\centering{}{\scriptsize{}0.001}} & \centering{}{\scriptsize{}4628.31\textpm 891.77} & \centering{}{\scriptsize{}5214.21\textpm 1126.09} & \centering{}{\scriptsize{}4926.53\textpm 1027.08} & \centering{}\textbf{\scriptsize{}6214.59\textpm 791.73} & \centering{}{\scriptsize{}5663.66\textpm 1134.54}\tabularnewline
\cline{3-7} \cline{4-7} \cline{5-7} \cline{6-7} \cline{7-7} 
 &  & \centering{}{\scriptsize{}0.027} & \centering{}{\scriptsize{}0.095} & \centering{}{\scriptsize{}0.048} & \centering{}{\scriptsize{}-} & \centering{}{\scriptsize{}0.504}\tabularnewline
\bottomrule

\end{tabular}
    \caption{Comparisons of algorithms on Noisy Ant-v2 tasks with different maximum episodic length setup. The averaged performance and standard deviation of 10 runs are reported, as well as the p-value that indicates the statistical significance of the difference between the proposed algorithm and each baseline algorithm. The best values of each row are shown in bold.}
    \label{tab:noisy_ant}
\end{table*}

In Section~\ref{sec:rq2}, we demonstrate the significant advantage of OVD-Explorer over DSAC and DOAC in stochastic Mujoco tasks. To further empirically verify the efficiency of OVD-Explorer, we detail the evaluation on Noisy Ant-v2 task with different maximum episodic length setup. The results presented in Figure~\ref{fig:ablation-horizon} and Table~\ref{tab:noisy_ant} illustrate that OVD-Explorer consistently outperforms the baseline algorithms across various maximum episodic lengths. Notably, longer maximum episodic lengths present higher task difficulty, particularly for high-dimensional tasks that demand more thorough exploration. From these experiments, we draw two key conclusions:

Firstly, these results further demonstrate the effectiveness of OVD-Explorer in exploring noisy environments. Figure~\ref{fig:ablation-horizon} shows the training curves on Noisy Ant-v2 tasks, including the error bars of interquartile range. Tab.~\ref{tab:noisy_ant} reports the average results and standard deviation, and p-values to show if a significant difference exists comparing baselines. Notably, most p-values comparing OVD-Explorer's performance to baseline algorithms are below 0.05, confirming that OVD-Explorer consistently and efficiently explores noisy environments across different levels of task difficulty.

Secondly, a practical finding is that exploration should be more conservative in harder tasks, where a smaller value of $\alpha$ in OVD-Explorer should be used. In Fig.~\ref{fig:ablation-horizon}(a), (b) and (c), $\alpha$ is set to 0.05 by default. 
In Fig.~\ref{fig:ablation-horizon}(d), (e), and (f), a smaller $\alpha$ is employed to achieve better performance in the more challenging tasks.

These empirical results further reinforce the superiority of OVD-Explorer in effectively addressing exploration challenges in noisy environments with varying degrees of difficulty. This makes it a practical and robust solution for continuous RL tasks with uncertainties and noise.

\begin{figure}[h]
\begin{center}

\subfigure[]{
\includegraphics[width=2.95cm]{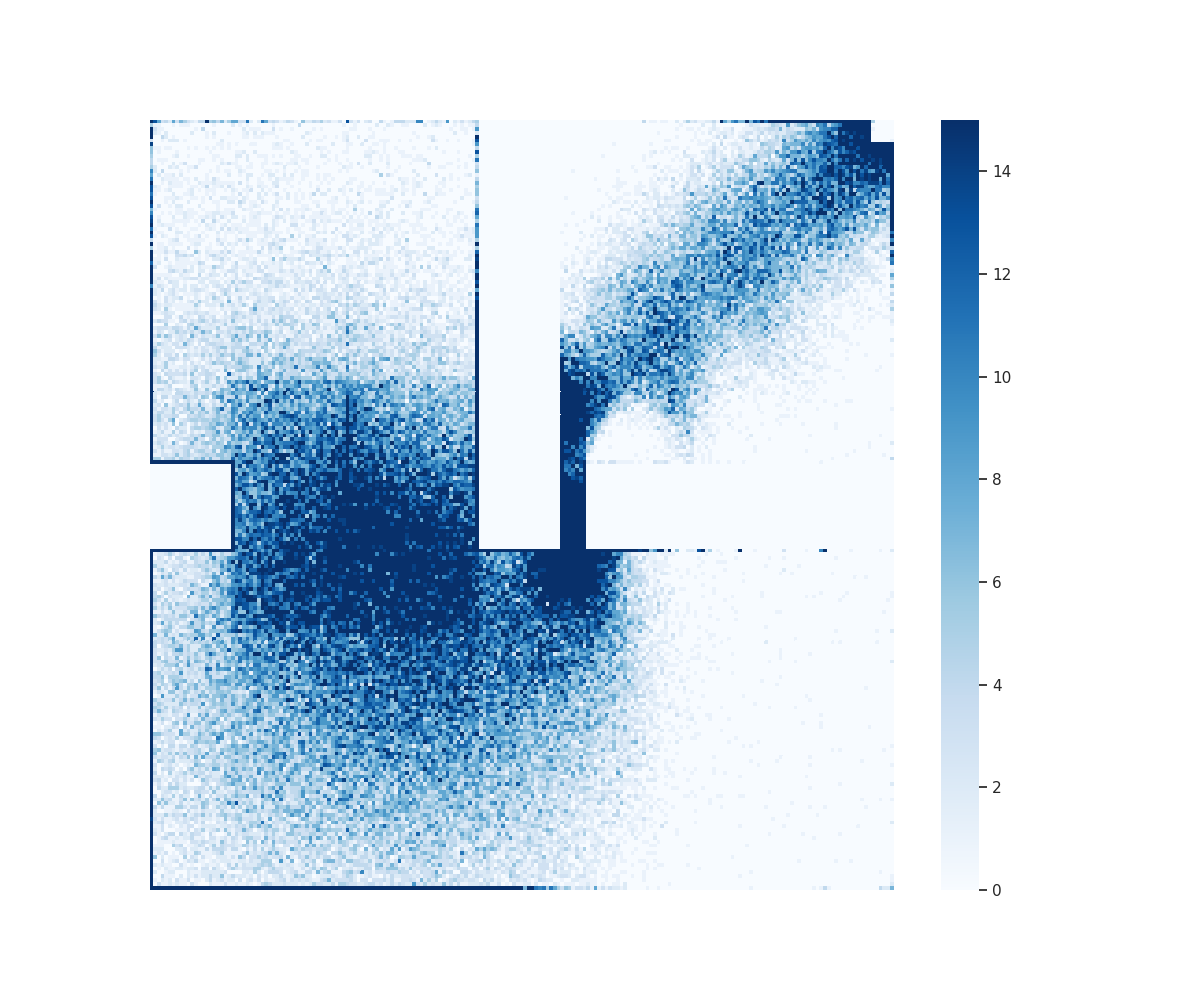}
\label{fig:m1}
}
\hspace{-7.5mm}
\subfigure[]{
\includegraphics[width=2.95cm]{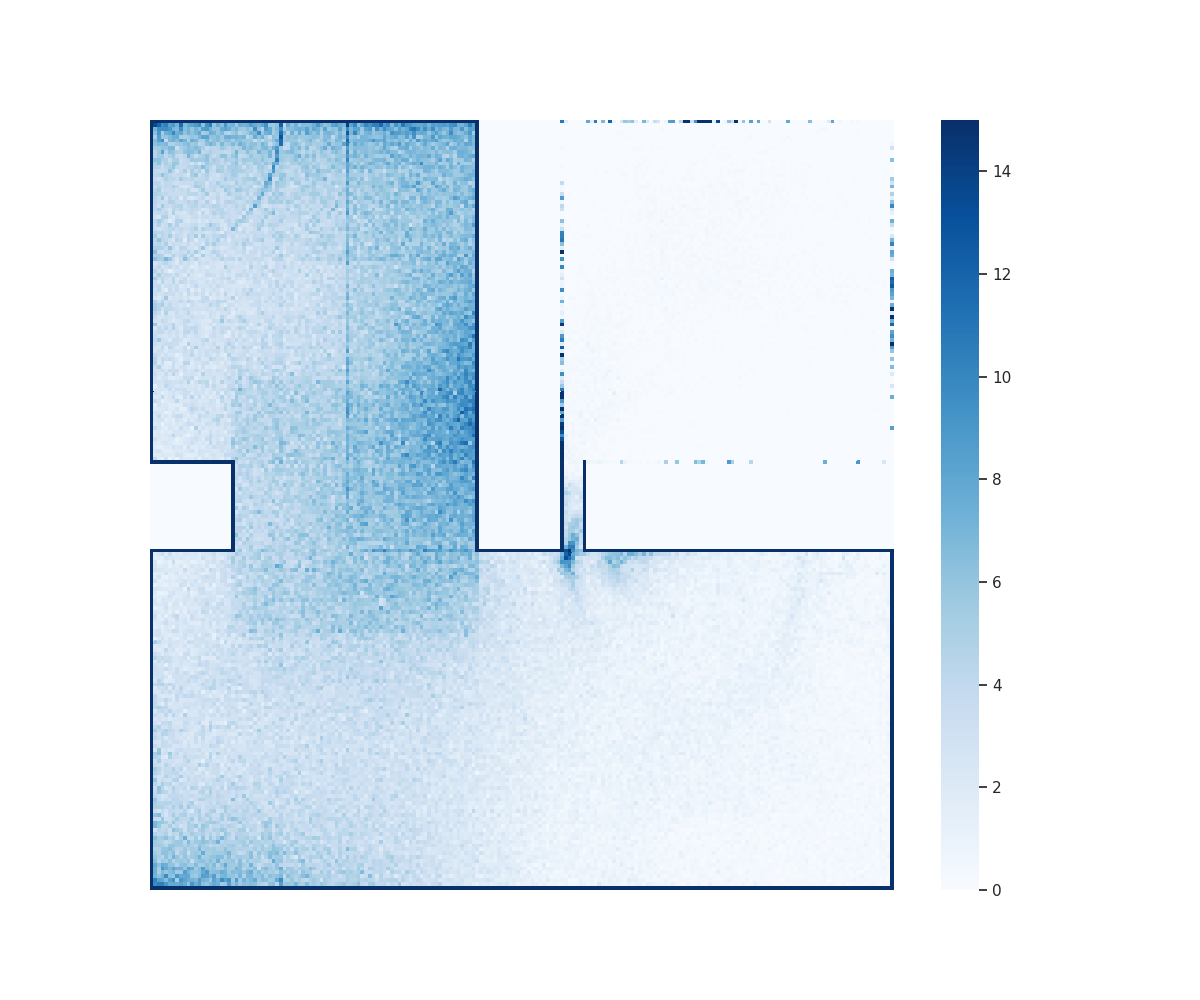}
\label{fig:d1}
}
\hspace{-7.5mm}
\subfigure[]{
\includegraphics[width=2.95cm]{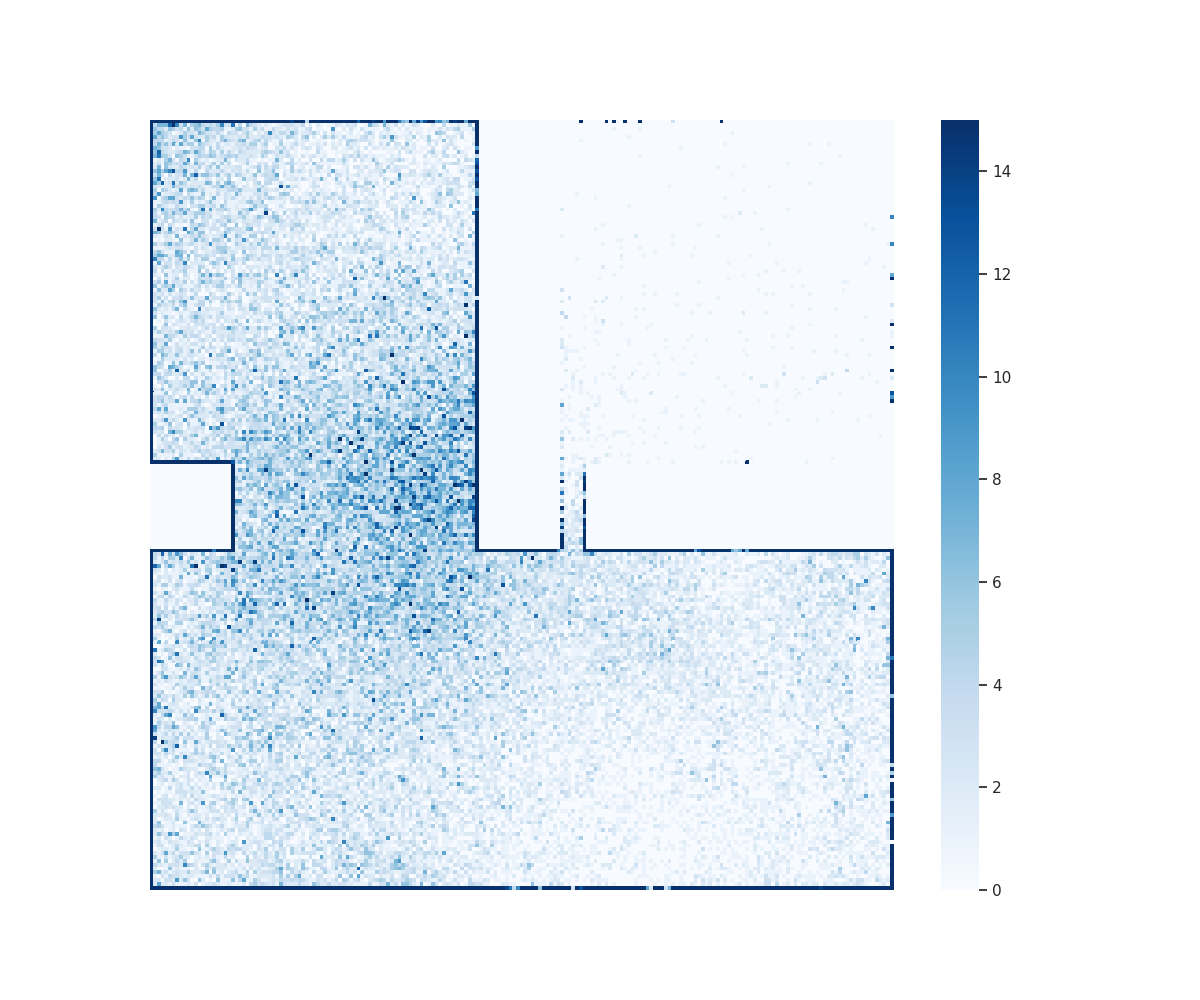}
\label{fig:o1}
}
\hspace{-7.5mm}
\subfigure[]{
\includegraphics[width=2.95cm]{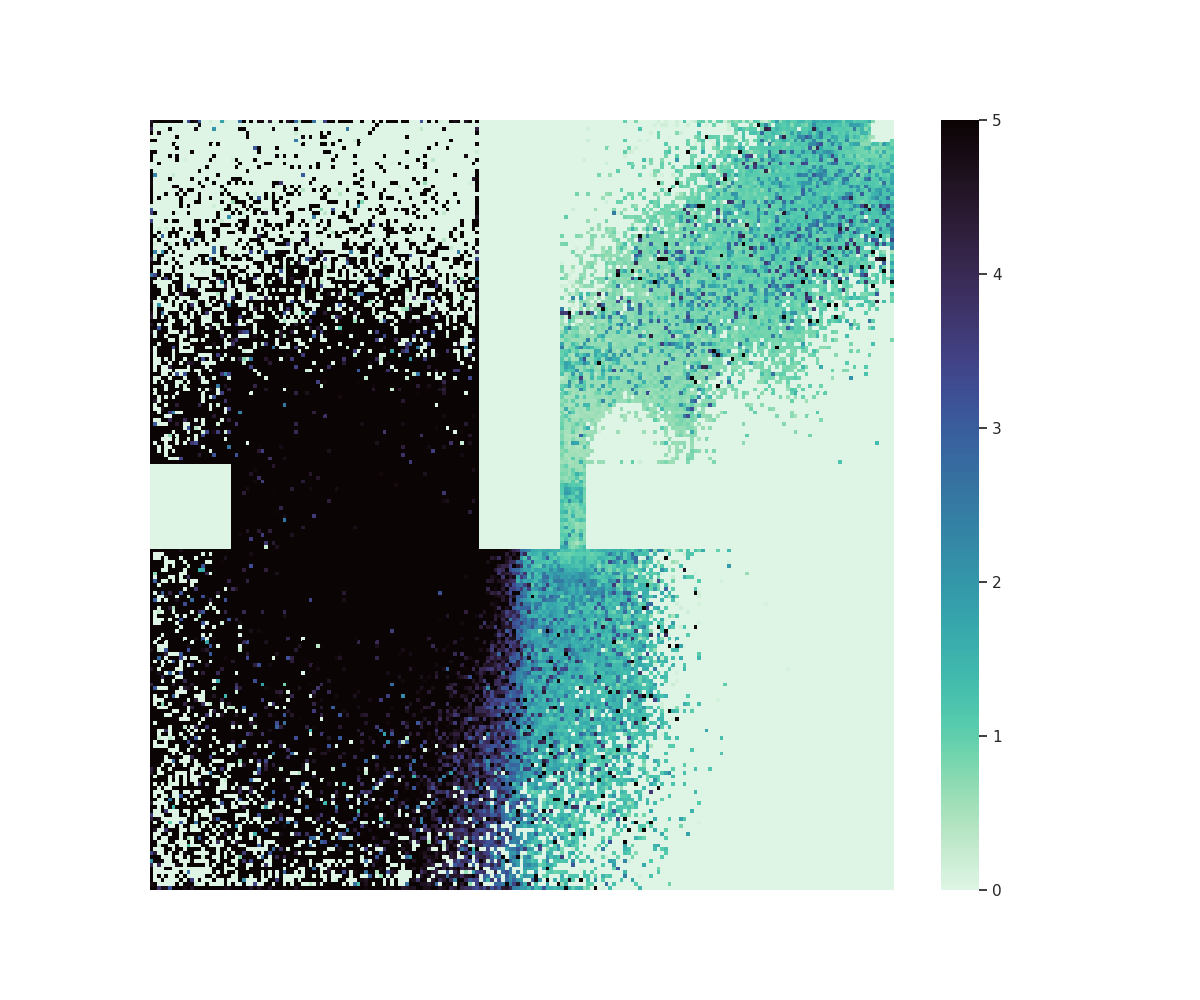}
\label{fig:mqes-aleatoric}
}
\hspace{-7.5mm}
\subfigure[]{
\includegraphics[width=2.95cm]{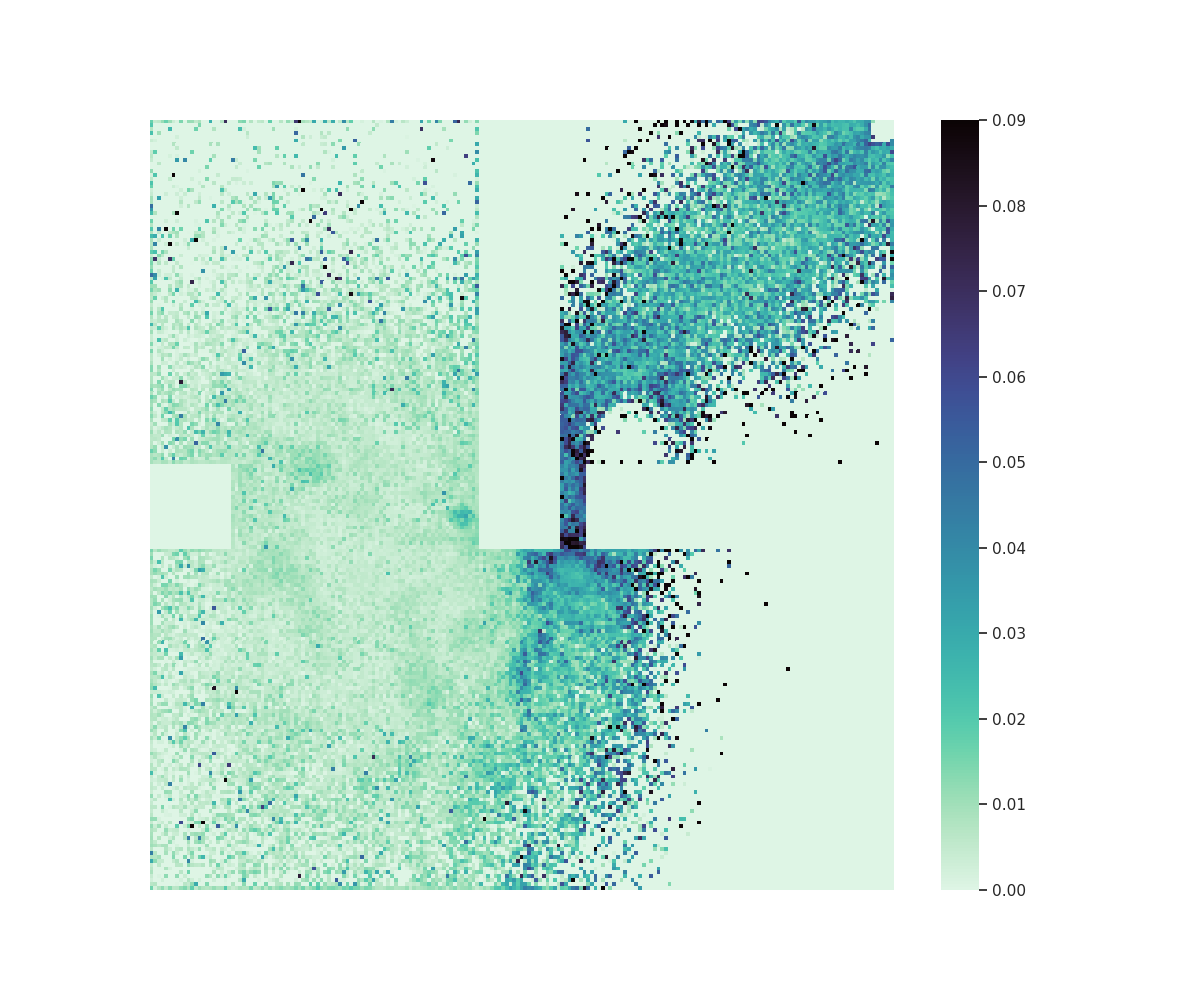}
\label{fig:mqes-ratio}
}
\hspace{-7.5mm}
\subfigure[]{
\includegraphics[width=2.95cm]{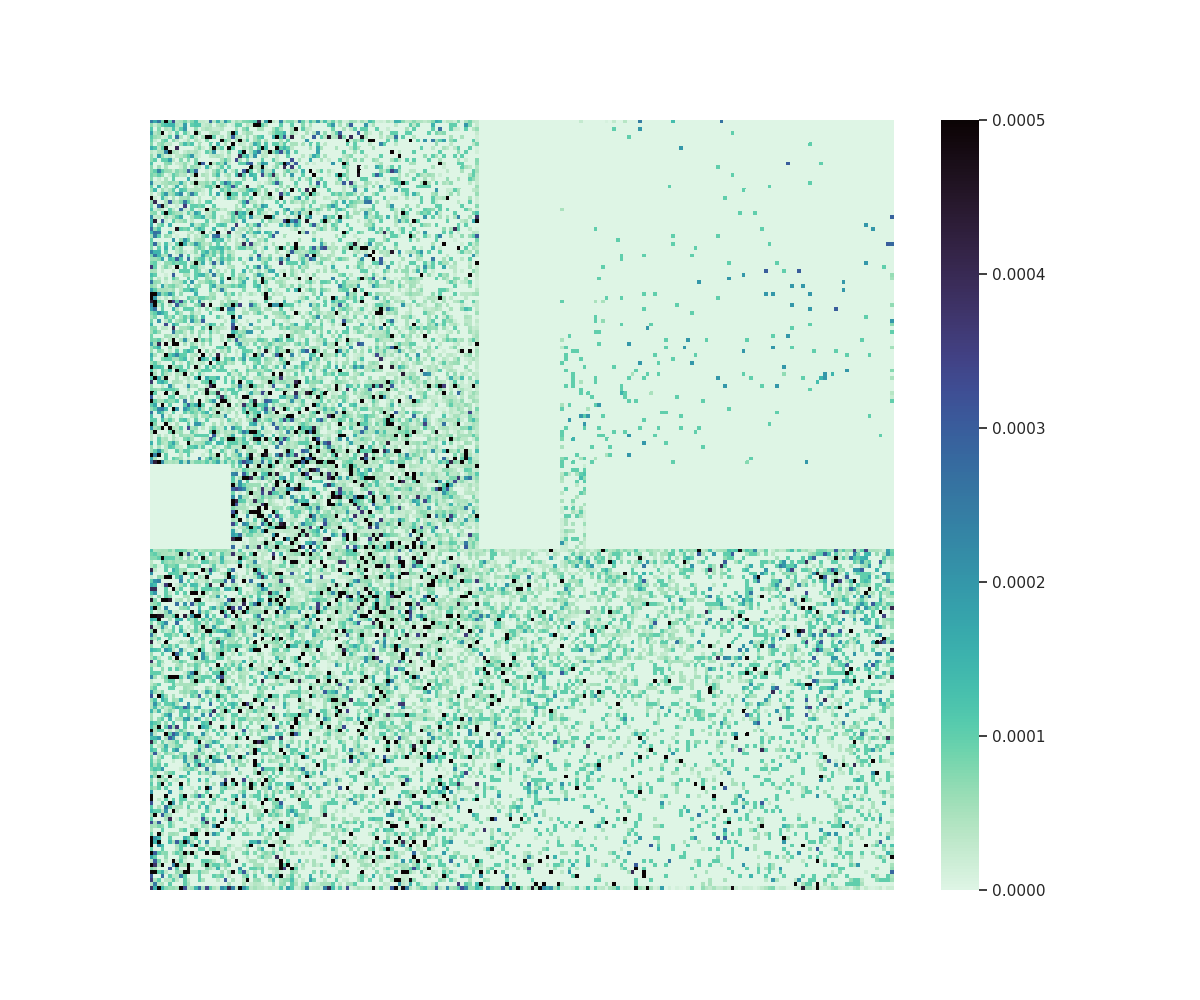}
\label{fig:doac-epistemic}
}
\label{fig:gridchaos-fres}
\caption{State visiting frequency heatmap from $1.0 \times 10^5$ to  $2.5 \times 10^5$ steps of one trial for (a) OVD-Explorer, (b) DSAC and (c) DOAC. (d) Estimated aleatoric uncertainty of OVD-Explorer; (e) Epistemic-aleatoric ratio of OVD-Explorer; (f) Estimated uncertainty for exploration in DOAC.
}
\end{center}
\end{figure}

\subsection{Exploration Patterns Visualization in the Case of GridChaos}
\label{app:analysis1}
In this section, we provide a visual analysis of the exploration patterns and advantages of OVD-Explorer compared to DSAC and DOAC in the GridChaos task.

To begin with, we present heatmaps showing the state visiting frequency during exploration for OVD-Explorer, DSAC, and DOAC in Figures~\ref{fig:m1},~\ref{fig:d1} and~\ref{fig:o1}, respectively.  From these heatmaps, we observe distinct exploration patterns. OVD-Explorer efficiently explores the right half of the environment, where the environmental risk is lower, while DSAC and DOAC are both trapped in the left half, characterized by higher risk. In contrast, Figure~\ref{fig:doac-epistemic} shows the estimated uncertainty in DOAC, which is larger on the left half. As DOAC encourages exploration in areas with relatively large estimated uncertainty, it explains why DOAC is stuck in the left half of the environment.

This visual analysis provides clear evidence of OVD-Explorer's ability to efficiently explore the environment by effectively balancing optimistic and noise-aware exploration based on accurate uncertainty estimations. The comparison with DSAC and DOAC further illustrates the significant advantage of OVD-Explorer in tackling exploration challenges in noisy environments.

\subsection{Analysis about Exploration Process of OVD-Explorer}
\label{app:analysis}
In this section, we conduct a statistical analysis to further verify that OVD-Explorer can effectively perform optimistic noise-aware exploration. We show the values of uncertainty estimations and our exploration objective (mutual information) at different stages during the training processes of two trials with different noise settings.

Figure~\ref{fig:heatmap1} illustrates the exploration process in a trial with lower environment noise around the goal, noting that the darker background color in the map represents higher noise, and the red dot represents the coordinate of the current state. The performance under this trail is given and the agent hardly ever reaches the goal before the 1000th epoch. Therefore, in the early stage, the aleatoric uncertainty is inaccurate and remains very low, as the value distribution shows little divergence. The figure also shows that our exploration objective (in green) is high when the epistemic uncertainty is high. So before the 1000th epoch, the exploration is guided by epistemic uncertainty, which follows the OFU principle.
As the agent explores further and begins to properly model the aleatoric uncertainty, i.e., the aleatoric uncertainty towards left is larger than right at current state (see epoch 1240). Then the mutual information value towards left is lower than right, although the epistemic towards left is higher. It indicates that OVD-Explorer can property guide noise-aware optimistic exploration. 

Figure~\ref{fig:heatmap2} shows the exploration process in a trial with higher environment noise around the goal. Similarly to the previous case, the early stages of exploration are guided by epistemic uncertainty. The agent would hardly estimate the accurate aleatoric uncertainty without obtaining any reward. Later, as the agent explores more and gains a better understanding of the environment, at the 1240th epoch, OVD-Explorer suggests exploring to the right, even though it has been recognized that the noisy on the right is high. This is because the epistemic uncertainty dominates under the mutual information at that time. In contrast, at the 1249th epoch, when the action towards right has been explored much, the significant higher aleatoric uncertainty towards right dominates. Therefore, following the mutual information, the action towards left is preferred. This demonstrates the trade off that OVD-Explorer make, performing noise-aware optimistic exploration.

\begin{figure*}[h]
    \centering
    \includegraphics[width=0.9\linewidth]{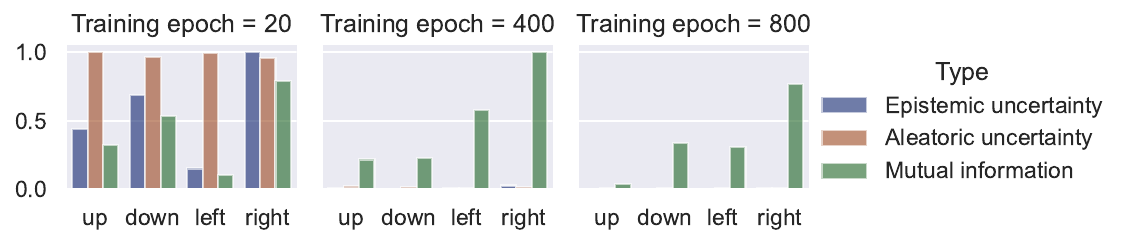}
\end{figure*}
\vskip -2.0cm
\begin{figure*}[h]
    \centering
    \includegraphics[width=0.9\linewidth]{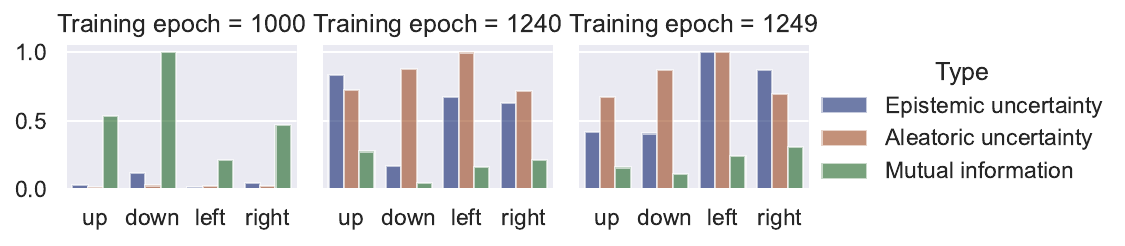}
\end{figure*}
\vskip -2.0cm
\begin{figure*}[h]
  \begin{minipage}[H]{0.72\linewidth}
    \centering
    \includegraphics[height=3.0cm]{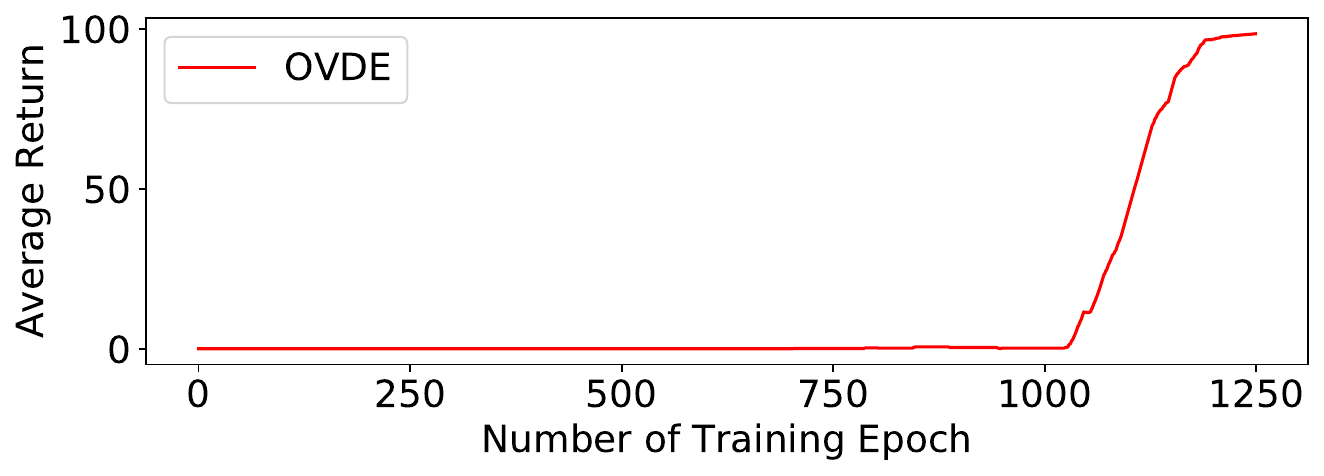}
  \end{minipage}%
  \begin{minipage}[H]{0.2\linewidth}
    \centering
    \includegraphics[height=2.75cm]{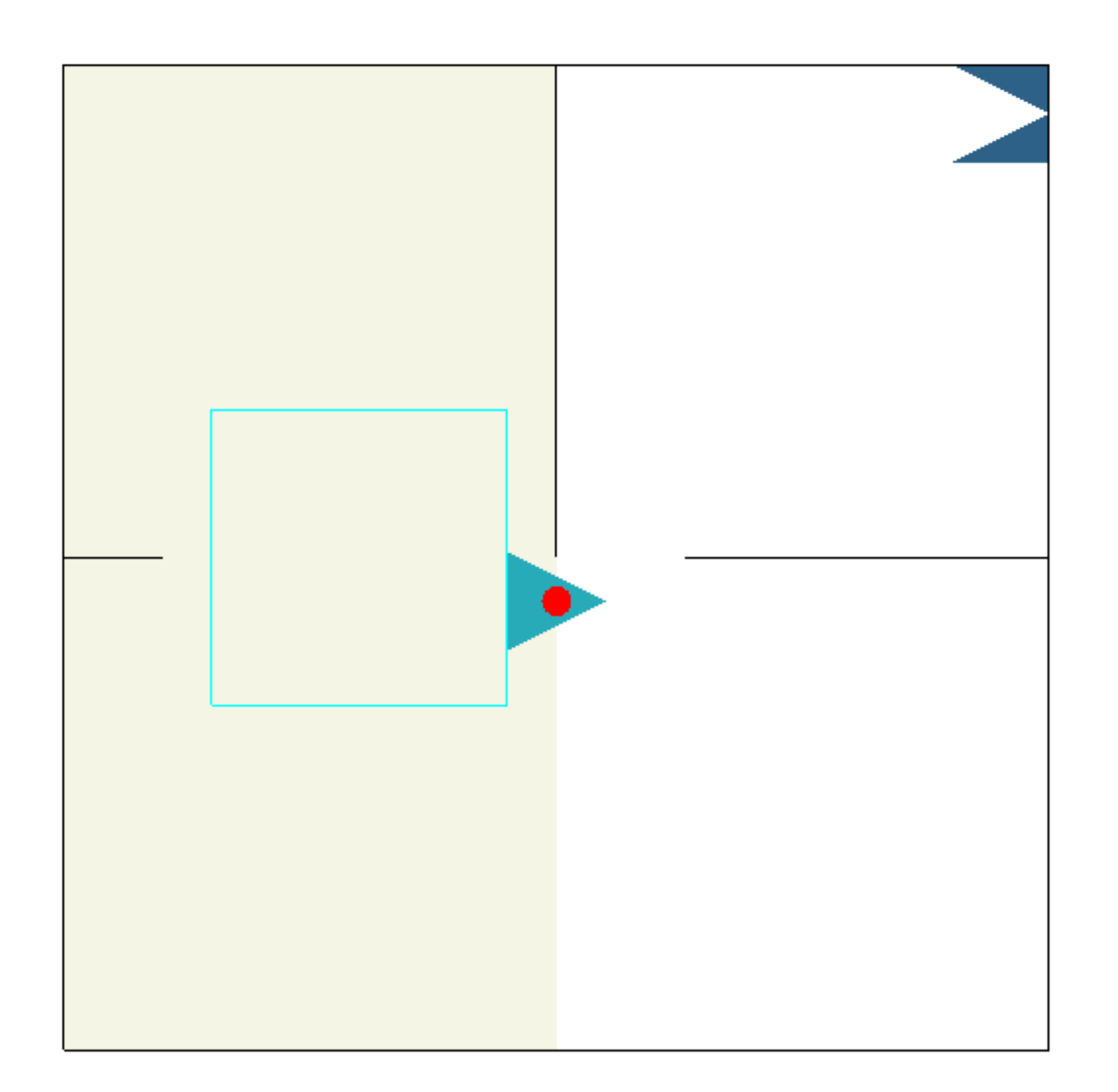}
  \end{minipage}
\vskip -0.1cm
\caption{The statistical analysis for the training process, with the aleatoric uncertainty around the goal is set lower.}
\label{fig:heatmap1}
\end{figure*}
\vskip -0.5cm
\begin{figure*}[h]
    \centering
    \includegraphics[width=0.9\linewidth]{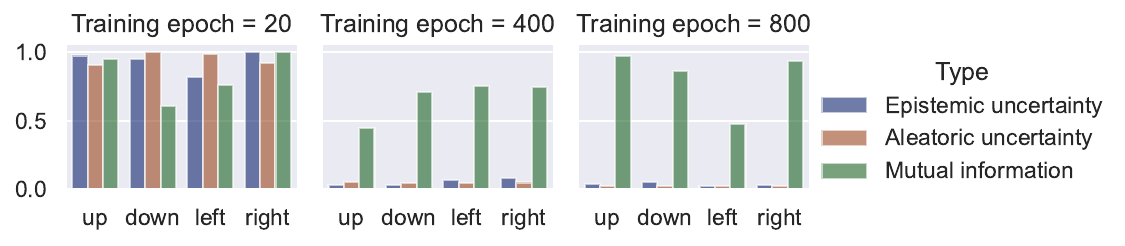}
\end{figure*}
\vskip -2.0cm
\begin{figure*}[h]
    \centering
    \includegraphics[width=0.9\linewidth]{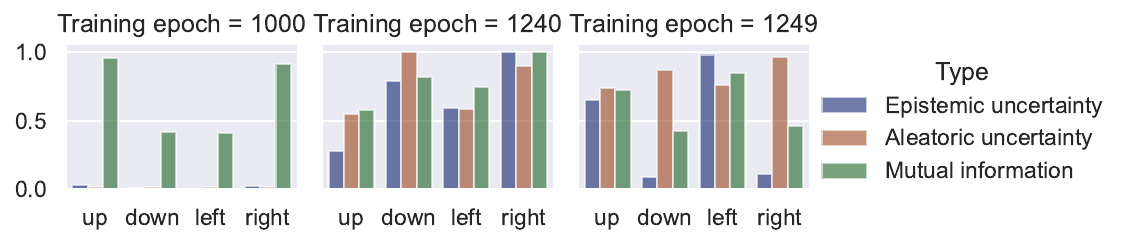}
\end{figure*}
\vskip -2.0cm
\begin{figure*}[h]
  \begin{minipage}[H]{0.72\linewidth}
    \centering
    \includegraphics[height=3.0cm]{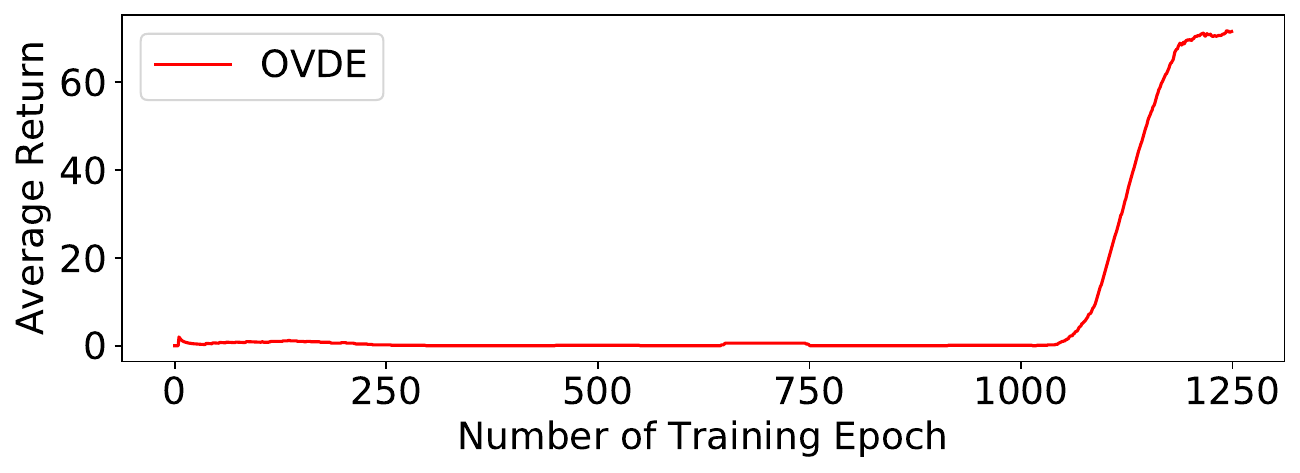}
  \end{minipage}%
  \begin{minipage}[H]{0.2\linewidth}
    \centering
    \includegraphics[height=2.75cm]{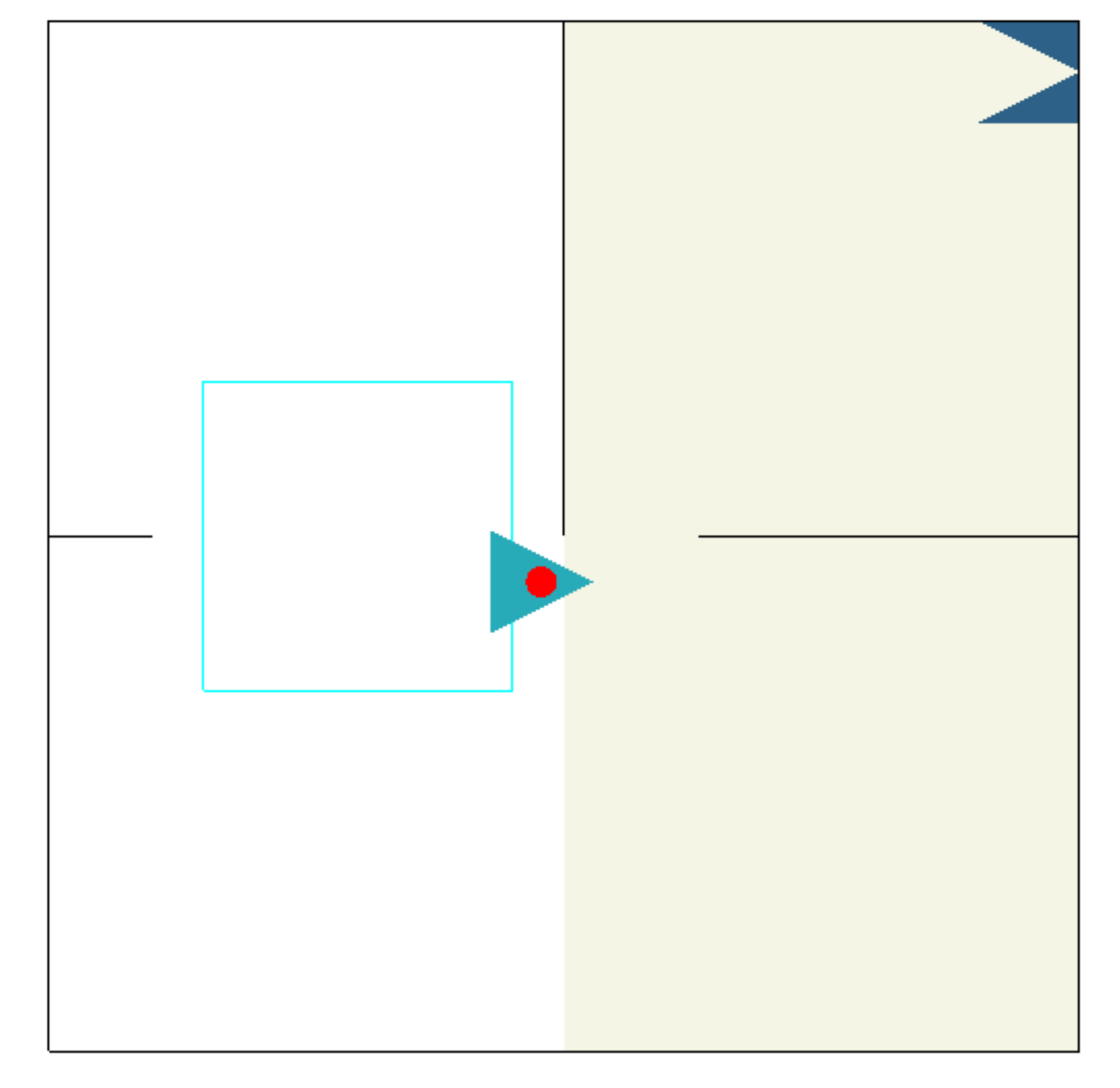}
  \end{minipage}
\vskip -0.1cm
\caption{The statistical analysis for the training process, with the aleatoric uncertainty around the goal is set higher.}
\label{fig:heatmap2}
\end{figure*}
\end{document}